\documentclass{article}

\usepackage{microtype}
\usepackage{graphicx}
\usepackage{subfigure}
\usepackage{booktabs} 

\usepackage[accepted]{icml2023}
\usepackage{cite}
\usepackage{bm}
\usepackage{verbatim}
\usepackage{authblk}
\usepackage{enumitem}
\usepackage{amsmath,amssymb,amsthm,amsfonts,physics,booktabs,multirow,multicol, url}
\usepackage{algpseudocode}  
\usepackage[ruled]{algorithm2e}
\usepackage{graphicx}
\usepackage{textcomp}
\usepackage{xcolor}
\newcommand{\eat}[1]{}

\newtheorem{lemm}{Lemma}

\newtheorem{prop}{Proposition}
\newtheorem{defi}{Definition}[section]
\newtheorem{theo}{Theorem}[section]
\newtheorem{remark}{Remark}

\newlist{Properties}{enumerate}{2}
\setlist[Properties]{label=Property \arabic*., font=\textbf, itemindent=*}

\usepackage{color,soul,xspace}
\newcommand{\todo}[1]{\textcolor{red}{[TODO: #1]}}

\newcommand{\cbit}{\begin{compactitem}}
	\newcommand{\ceit}{\end{compactitem}}
\newcommand{\cben}{\begin{compactenum}}
	\newcommand{\ceen}{\end{compactenum}}

\newcommand{\mourmeth}{\text{CCVGAE}}
\newcommand{\ourmeth}{$\mourmeth$\xspace}

\usepackage{color,soul,xspace}
\def\BibTeX{{\rm B\kern-.05em{\sc i\kern-.025em b}\kern-.08em
		T\kern-.1667em\lower.7ex\hbox{E}\kern-.125emX}}

\icmltitlerunning{Concept-free Causal Disentanglement with Variational Graph Auto-Encoder}

\begin{document}
	
	\twocolumn[
	\icmltitle{ Concept-free Causal Disentanglement with Variational Graph Auto-Encoder}
	
	
	
	\icmlsetsymbol{equal}{*}
	
	\begin{icmlauthorlist}
		\icmlauthor{Jingyun Feng}{to}
		\icmlauthor{Lin Zhang}{goo}
		\icmlauthor{Lili Yang}{to}

		\icmlaffiliation{to}{Department of Statistics and Data Science, Southern University of Science and Technology, Shenzhen, Guangdong, China.}
		\icmlaffiliation{goo}{Gongsheng Matrix, Shenzhen, China}
	\end{icmlauthorlist}

	
	
	
	\vskip 0.3in
	]
	
	
	
	\printAffiliationsAndNotice{}  
	
	\begin{abstract}
		
		In disentangled representation learning, the goal is to achieve a compact representation that consists of all interpretable generative factors in the observational data.
		Learning disentangled representations for graphs becomes increasingly important as graph data rapidly grows.
		Existing approaches often rely on Variational Auto-Encoder (VAE) or its causal structure learning-based refinement, which suffer from sub-optimality in VAEs due to the independence factor assumption and unavailability of concept labels, respectively. 
		In this paper, we propose an unsupervised solution, dubbed concept-free causal disentanglement, built on a theoretically provable tight upper bound approximating the optimal factor.
		This results in an SCM-like causal structure modeling that directly learns concept structures from data.
		Based on this idea, we propose Concept-free Causal VGAE (CCVGAE) by incorporating a novel causal disentanglement layer into  Variational Graph Auto-Encoder.
		Furthermore, we prove concept consistency under our concept-free causal disentanglement framework, hence employing it to enhance the meta-learning framework, called  concept-free causal Meta-Graph (CC-Meta-Graph).
		We conduct extensive experiments to demonstrate the superiority of the proposed models: CCVGAE and CC-Meta-Graph,  reaching up to $29\%$ and $11\%$ absolute improvements over baselines in terms of AUC, respectively. 
		
	\end{abstract}

\section{Introduction}

Graph data becomes ubiquitous, in both natural and human-made scenarios, along with the rise of deep learning, gaining increasing attention and making learning on graphs an emerging research field, with the goal of understanding graphs and dealing with downstream applications, such as drug discovery~\citep{Jiaxuan2018}, traffic forecasting~\citep{Jiang2022}, recommender systems~\citep{Shiwen2020}, and others~\citep{zhou2020graph}. 
Of particular importance is graph representation learning~\citep{hamilton2020graph}, but it remains an outstanding research problem due to graphs' non-IID and non-Euclidean properties. 
There is growing attention to the disentanglement learning  to address this problem.

\begin{figure}[t]
    \centering
    \includegraphics[trim={0cm 0.1cm -3cm 1cm},clip,width=0.75\textwidth]{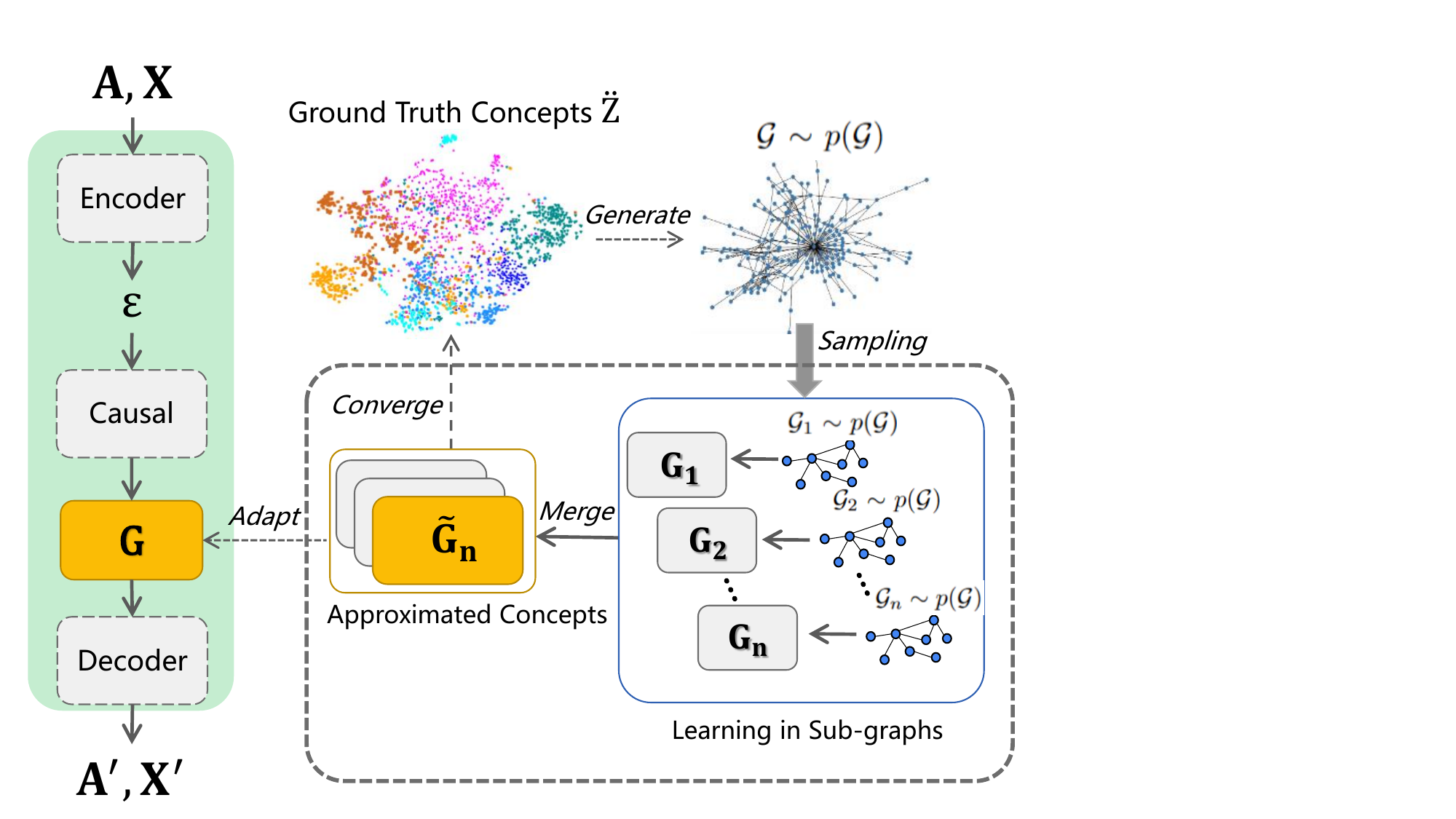}
    \caption{
     An illustration of our proposed ideas. 
    The left component (shown in light green) demonstrates a causal disentangle process in a single graph, i.e., CCVGAE which takes an adjacency matrix and node features as inputs.    
 The right-hand component shows the consistency property of generative factors obtained by CCVGAE.
 This property implies that merging generative factors from other graphs can be adapted to others, and thus leads to an extension of CCVGAE, called CC-Meta-Graph. Here, we assume all graphs are sampled from the same graph (i.e., $\mathcal{G}$). $\bf \Tilde{G}_n$  represents the approximated concepts by merging the individual concepts ($\mathbf{G}_i$) obtained from the samples ($\mathcal{G}_i$).}
    \label{fig:intro_overview}
\end{figure}


The goal of disentanglement learning is to acquire the representations that capture all  interpretable generative factors,
called disentangled representations~\citep{bengio2013representation,higgins2018towards}.
A  significant challenge of disentanglement learning is that we often only have raw observations while not allowing any supervision on generative factors
(i.e., causes)~\citep{kumar2017variational}.
Earlier attempts~\citep{paige2017learning, yang2021causalvae} often demand adequate labels for training and hence  can not fit the above realistic setting.
This motivates us to focus on the unsupervised setting.
Recent advances in unsupervised disentanglement learning have mostly focused on Variational Auto-encoders (VAEs)~\citep{Li2018DisentangledSA} and Generative Adversarial Networks (GANs)~\citep{kocaoglu2017causalgan}.
In particular, the VAE framework is preferred in graphs because of its stability in contrast to mode collapse in GANs due to its implicit modeling of the distribution, which is especially difficult to learn the distribution of graphs.
So in this work, our focus is on the VAE framework to explore disentanglement for graph representation learning, i.e.,  Variational Graph Auto-Encoders (VGAE)~\citep{Thomas2016}. 



Despite the recent growth of disentanglement learning, most state-of-the-art methods within the VAE framework have assumed that the distributions in the hidden space are independent Gaussian~\citep{kim2018disentangling} and thus lead to suboptimal solutions~\citep{trauble2021disentangled}.
Studies~\citep{locatello2020commentary,Trauble2020OnDR} have shown that disentanglement of representations is nearly impossible under the independent assumption when the data demonstrates intrinsic correlations.
In contrast, modeling the structure for underlying factors enhances disentanglement, particularly causal structure learning~\citep{scholkopf2022statistical}.
However, when leveraging the VAE framework, there is no adequate research on the optimal solution while imposing a causal structure on the latent factors.

\eat{
underlying generative factors (optimal)

Together with the VAE framework for graph representations, 

leads to the challenge of ensuring the optimality of latent factors while imposing causal structure on them. 

We need to answer such question: How could causality enhance disentanglement learning, and to what extend can we approach the optimal disentanglement, i.e., acquire underlying factors?

However, only in a few cases do we observe clear causalities, such as the speed and gravity as in~\citep{yang2021causalvae}, making SCM impractical, given its need for pre-defined concepts and their labels.\textcolor{blue}{It is also problem about unsupervised or supervised, mentioned in last paragraph. Suggest to skip this sentence.}
In addition, leveraging the VAE framework for graph representations leads to the challenge of ensuring the optimality of latent factors while imposing causal structure on them. \textcolor{blue}{Too fact to "optimal". we can say: We need to answer such question: How could causality enhance disentanglement learning, and to what extend can we approach the optimal disentanglement, i.e., acquire underlying factors?}
}

In this paper, we attempt to address {unsupervised} causal disentangled representation learning in the VGAE framework, including theoretical analysis and practical methodologies.
We prove a tight upper bound on approximating the optimal latent factor via causal structure learning.
It indicates that a linear causal modeling function can approximate the optimal latent factor with high confidence.
With this, we then develop a practical causal disentanglement method without requiring concept labels, called concept-free causal disentanglement.
In this way, we achieve a data-driven causal structure modeling that directly learns concept structures from data.
Building on this, we introduce a novel causal disentanglement layer and then integrate it with VGAE, resulting in our first model, called Concept-free Causal VGAE (CCVAGE).
Besides, we uncover the consistency of our obtained concepts due to the data-driven style, making them suitable for capturing underlying global information with little data.  
Towards this, we propose a meta-learning model that transfers global-aware concepts to newly arrived data, resulting in our second model, called CC-Meta-Graph.
In Figure~\ref{fig:intro_overview}, we present an illustration of the proposed ideas.

We highlight the contributions of this paper:\\ 
1,  In this paper, we theoretically prove a tight bound on the approximation of optimal factors and offer a practical causal disentanglement method on top of it, called concept-free causal disentanglement.
\\
2, We propose two causal disentanglement-enhanced models: one is to support causal disentanglement in VGAE, and the other is to validate the proposed consistency property.
\\
3, We conduct extensive experiments with synthetic and real-world graph data to demonstrate the efficiency of our proposed models in terms of link prediction, achieving up to $29\%$ and $11\%$ absolute improvements for CCVGAE and CC-Meta-Graph, respectively. \footnote{The experiments code can be found in\\ $https://www.dropbox.com/sh/c8nd1qbpb20ling/AABhhj\\rlRGOF4X5h-osw_0aza?dl=0$}

\eat{
In this paper, we attempt to consider both the unavailability of labels and the richness of the underlying latent factors' structure to scale well for realistic settings.
A practical approach to modeling the underlying factor structure for disentanglement is to leverage causal structure learning due to its well-established capabilities~\citep{}.
Often, the Structural Causal Model (SCM)~\citep{shimizu2006linear}, is employed for modeling causal structures.
However, only in a few cases do we observe clear causalities, such as the speed and gravity as in~\citep{yang2021causalvae}, making SCM incompetent, given its need for pre-defined concepts and their labels.
In addition, leveraging the VAE framework for graph representations leads to the challenge of ensuring the optimality of latent factors while imposing causal structure on them.
}

\eat{
A practical approach to handling the structure of the underlying factors is the Structured Causal Model\textcolor{blue}{Jump too fast. I suggest "...factors is causal model, especially SCM..." } (SCM)~\citep{shimizu2006linear}, which offers rigorous disentangled representations for existing models \textcolor{blue}{ rigorous?}~\citep{yang2021causalvae,suter2019robustly}.
However, only in a few cases do we observe clear causalities, such as the speed and gravity as in~\citep{yang2021causalvae}, making SCM incompetent when having no access to pre-defined concepts and labels.
}

\eat{
Given the above findings, how might one obtain the disentangled representation for VAGEs from data without independent underlying factors?
Our key idea is to model non-independent underlying factors with causality to ensure disentanglement in the representation.
As a commonly used causality model, the Structured Causal Model (SCM)~\citep{shimizu2006linear} opens an avenue to effectively capture structures in factors, along with pre-defined concepts and exogenous factors, offering rigorous disentangled representations for existing methods~\citep{yang2021causalvae,suter2019robustly}.
But the above question is far from being solved by these SCM-based methods, given the challenge of defining concepts in advance.
Only in a few cases do we observe clear causality, such as the speed and gravity as in~\citep{yang2021causalvae}, making SCM an impractical method to accomplish causality and hence to perform disentanglement.
Furthermore, another drawback in this context is the requirement for supervision~\citep{yang2021causalvae, locatello2020commentary}.
Nonetheless, supervised information is generally too expensive to be available, making unsupervised disentanglement necessary.
}

\eat{
Can we instead use concepts without priors but learn them directly from data? 
In other words, we can simplify our disentanglement learning problem via a data-driven style, relaxing the problem to latent structure learning with SCM-alike disentangled constraints. 
For this purpose, we reformulate the supervised information, i.e., the exogenous factors in SCM, \textcolor{red}{as a learnable parameter}, 
and use \textcolor{red}{the endogenous// wrong} to infer the causal structure of the latent factors, 
called \textit{concept-free SCM}.
\todo{concept definition:  目前我们学习到的是一种高纬度表征，而非显示的定义具备物理意义的 外生变量，所以我们学习到的 不能叫 concept，但是具备concept在模型中的作用，只是无法人类解释}
Following, how can we perform unsupervised disentangled representation learning?  
As supervision alternatives, we make a mild assumption that the latent space of observations and causal disentangled representations share the same distribution, ensuring consistency in representation learning.
}

\eat{
Our primary contribution in this work is a causal disentangled representation learning layer for the framework of VAGE that enables disentangling representations without supervision, with the proposed concept-free SCM and consistent representation learning..........................
}



\eat{
\begin{figure}
    \centering
    \includegraphics[width=0.40\textwidth]{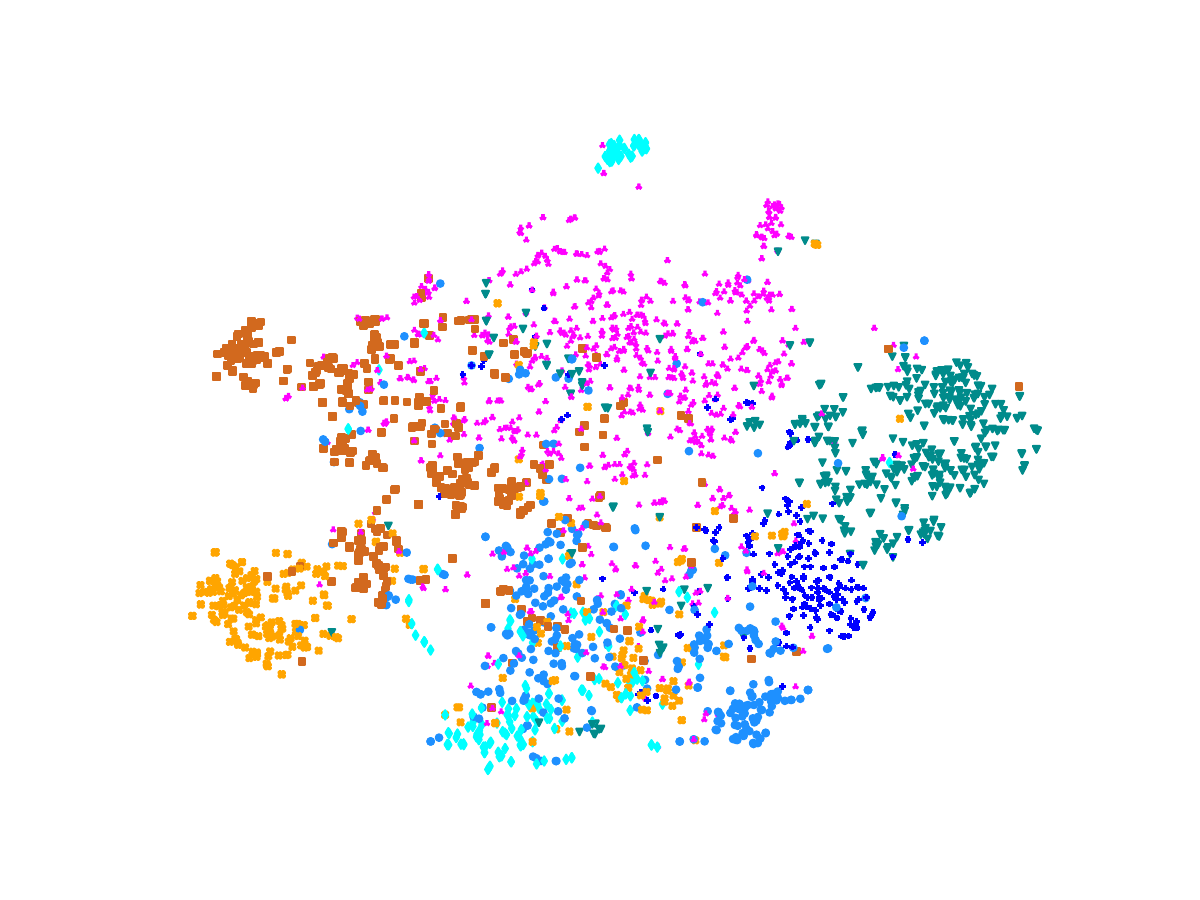}
    \caption{Visualization of latent representation of CCVGAE for Cora. Each color denotes a document class, which is unknown in training. The same class nodes have similar latent representation, suggesting that we acquire underlying concept "class".}
    \label{fig:synthetic AUC}
\end{figure}
}

\eat{
Analyzing graph data is an important machine learning task with variety of applications. Traffic network, social network and recommendation system are datasets which can be modeled as graph. Each node in graph represents an agent (for example, people, and genes), and edge represents the interaction between agents. The main challenge of link prediction analysis, clustering, or node classification of graph data sets is how to deploy the structural information in the graph model. The graph representation learning aims to sum up the feature vector of image structure information in the low dimensional space, which can be used for downstream analysis tasks.

Most methods assume that each node can be embedded into a certain potential space \citep{belkin2001laplacian}, \citep{ahmed2013distributed}, \citep{tang2015line}, \citep{armandpour2019robust}, but modeling uncertainty is crucial in many applications. To solve this problem, the variational image autoencoder (VGAE) \citep{kipf2016variational} embeds each node into the space of random variables. Although it is very popular, 

1) the assumption of independent Gaussian distribution in hidden space is not common in real data sets. For example, in medical data sets, genes are often an important factor in the correlation between diseases, but different genes are often related, which leads to that the connection probability between disease nodes does not simply depend on the product of gene marginal distribution. 

2) Some studies have pointed out that it is difficult to obtain the low dimensional disentangled representation, which makes the representation can not be used for a large variety of tasks and domains.

In the disentangled representation learning, some research \citep{trauble2021disentangled} formally described the case when data has correlation, and theoretically proved that in this case, the VAE framework could not find a independent distribution in latent space to make the posterior distribution equal to the real data, which is opposite from their assumption (or independent assumption will lead to sub-optimal solution). In recent years, some researches have combined causality and disentangled representation learning and proposed causal disentanglement\citep{suter2019robustly}. They try to model the correlation of the generation factors by causal learning methods but mostly stay in theory. 

Inspired by the causal disentanglement, for 1): we have proved that when the causal disentanglement process meets certain conditions, we can find the optimal low dimensional representation when data set has correlation between generator factors. For 2): we improve the VGAE framework based on this theory, and experiments show that the low dimension representation of one graph can be applied for other graphs.
}
	\section{Related work}

{\bf Disentanglement Learning.}
The concept of disentanglement was first introduced by ~\citet{bengio2013representation} as a property of representation and  its formal  definition~\citet{higgins2018towards} is: if a representation can be decomposed into several independent features, which means only one of these features will change when change one factor of data input, then we call it \textit{"disentangled representation"}. 
Some studies ~\citep{eastwood2018framework} consider a more rigorous definition that only if each dimension of the representation can capture at most one true generative factor, we can call this representation a "disentangle representation".
In order to encourage potential factors to learn disentangled representations while optimizing the inherent task objectives, disentangled representation learning is designed to capture interpretable, controllable and robust representations.

In graph disentangled representation learning, most frameworks are GNN-based. DisenGCN~\citep{ma2019disentangled} utilizes neighbourhood routing to identify the latent factor that may caused edges, FactorGCN ~\citep{yang2020factorizable} disentangle graph into several sub-graphs, each sub-graph represent graph composed of one type of edges. However, ~\citet{fan2022debiasing} noticed that GNNs always suffer from spurious correlation, even if the causal correlation always exists. They proposed DisC to learn causal substructure and bias substructure. DisC requires some of input graph nodes represent concepts. However, such graphs are always unavailable in real world.

Most typical disentangled representation learning methods are generative models, especially Variational Auto-Encoder(VAE)~\citep{higgins2016beta,kumar2017variational,yang2021causalvae,zhu2021commutative}.
VAE use a variational posterior $q(z|x)$ to approximate the unknown true posterior $p(z|x)$.To obtain better disentanglement ability, researchers design various extra regularizers based on the original VAE loss function. A penalty coefficient is introduced to ELBO loss by $\beta$-vae ~\citep{higgins2016beta} to strengthen the independence constraint of the variational posterior distribution $q (z|x)$. FactorVAE ~\citep{kim2018disentangling} imposes independence constraint according to the definition of independence. 
However, better disentanglement ability often leads to more reconstruction errors. 
To balance the trade-off between reconstruction and disentanglement, \citet{burgess2018understanding} proposes a simple modification based on $\beta$-VAE, making the quality of disentanglement can be improved as much as possible without too much reconstruction error. 
However, we believe that the conflict between reconstruction error and disentanglement quality does not naturally exist but from the improper disentanglement such as independent assumption as follows.  
~\citet{higgins2018towards} assumed that the generating factors were natural and independent in disentangled representation learning. 

However,  ~\citet {suter2019robustly} disagreed with the independence assumption. They assumed that the generating factors of the observable data are causally influenced by the group of confounding factors, and first introduced SCM~\citep {krajewski2010rh} to describe causal relationships among generating factors. ~\citet{trauble2021disentangled} suggested that, if some generating factors are correlated in the data set, methods based on independent assumption might have a bias against disentanglement. Other researchers ~\citep{yang2021causalvae,shen2022weakly} have also taken experiments of disentanglement learning in real-world data based on the assumption that the real-world data is not generated by independent factors. 
Here, we also follow the same assumption that generates factors are not independent and even believe there are underlining causal relationships among these factors.

{\bf Causal Disentanglement.}
Over the past decades, many researchers ~\citep{hoyer2008nonlinear,zhang2012identifiability,shimizu2006linear}  have paid attention to the discovery of causality from observational data.
With the development of disentanglement learning, the community has raised the interest in combining causality and disentangled representation. ~\citet{kocaoglu2017causalgan} proposed a method called CausalGAN which supports ”do-operation” on images but it requires the causal graph given as a prior. ~\citet{suter2019robustly} believed that the underlying causal generative process will impact the level of disentanglement, and firstly proposed the definition of the causal disentanglement process.

\citet{yang2021causalvae} is the first to implement the causal disentanglement process proposed by~\citet{suter2019robustly}, called CausalVAE.  However, their method is semi-supervised because they require labels of generative factors.
 But such labels are uneasily acquired in the graph, so our work concentrates on an unsupervised method, which makes latent variables learn underlying causal information from data.  


	\section{Notations and Preliminaries}

{\bf Notations.}
Formally, let  $\mathcal {G = (V, E)}$ denotes an undirected and unweighted graph, with its adjacency matrix and degree matrices as ${Adj}=\{0,1\}^{n \times n}$ 
and $D$, respectively. Here, $ \mathcal V $ and $\mathcal E$ are the node and edge sets, where $ n = |\mathcal V|$. 
Nodes are associated with pre-defined attributes, written as ${Attri} \in \mathbb{R}^{n}$.
We use $\mathcal{N}(\cdot)$ to represent Gaussian distribution.


{\bf Variational Graph Auto-Encoder (VGAE).}
In VGAE, the high-dimensional observation $\varepsilon$ is projected into a low-dimensional space for compact representation $z$ using an encoder-decoder framework.
Concretely, the encoder compresses the input data, and the decoder checks the soundness of the compressed one by recovering raw data.
Mathematically, we often optimize VAE by maximizing evidence lower bound (ELBO), denoting as 
$\max _{p, q} \mathcal{L}_{\text {VAE }}(\varepsilon), \text { where } \mathcal{L}_{\text {VAE}}(\varepsilon)=\mathbb{E}_{q(z \mid \varepsilon)}[\log p(\varepsilon \mid z)]-D_{\mathrm{KL}}(q(z \mid \varepsilon) \| p(z))$, where $D_{\mathrm{KL}}$ is the Kullback-Leibler divergence~\citep{joyce2011kullback}. 
\textbf{VGAE is a special cases of VAE}, with graph-based functions for $q(\cdot)$ and $ p(\cdot)$. 
Using graph convolutional network (GCN)~\citep{zhang2018graph}, we have them defined as follows: $p\left(Adj_{i j}=1| \varepsilon _i, \varepsilon _j\right)=\sigma\left(\varepsilon _i^{\top} \varepsilon _j\right)$ and $q\left ( \varepsilon_i |Adj,Attri\right ) = N\left  (\mu_i , diag(\sigma_i)\right )$, where 
the  mean is 
$\mu = GCN_\mu(Adj,Attri)$ and  covariance is $\sigma = GCN_{\sigma}(Adj,Attri)$.
Here, $\sigma(\cdot)$   is an activation function and  takes the logistic Sigmoid function by default.

\eat{ We are given an undirected, unweighted graph $\mathcal {G = (V, E)}$ with $ n = |\mathcal V|$ nodes with features. Denote the graph adjacency matrix as a $n \times n$ matrix $A$ and the node features as a $n \times d$ matrix $X$. Denote the degree matrix of $A$ as $D$.

The process of Variational Graph Auto-Encoders includes Encoder (Inference model), which aims to get latent representation in low dimension space, and Decoder (Generative model), which aims to reconstruct adjacency matrix $A$. These two models can be described as follows:\\

(1) Encoder:\\
We denote latent variables as $K$ dimensions vectors $\varepsilon _i, i=1,2,...,n$, which can be summarized in a $n \times K $ matrix $\bm{\varepsilon}$. Assume that each $\varepsilon _i$ is independent normal random vector, $ i=1,2,...,n$.

Denote covariance and mean matrix of $\bm{\varepsilon}$ are $\sigma_{n \times K\times K}$ and $\mu_{n \times K}=\{\mu_i\}_{i=1,2,...,n}$ , 
in which $\mu_i$ is $K$ dimensions vector. Because each $\varepsilon _i$ is independent, the covariance matrix is diagonal matrix, then $\sigma$ can be represented by $\sigma=\{{diag(\sigma_i)_{K \times K}}\}_{i=1,2,...,n}$, in which $\sigma_i$ is $K$ dimensions vector. We denote $\{\sigma_i\}_{i=1,2,...n}=\sigma'$, then $\sigma'$ is a $n\times K$ matrix, which could absolutely determine $\sigma$. 

In Variational Graph Auto-Encoders, two layers GCN are adopted to get $\mu$ and $\sigma'$: 
 \begin{equation}
 \label{two GCN}
\begin{aligned}
\mu = GCN_\mu(A,X)=A_{norm} ReLU(A_{norm} W_0 X) W_1\\
\sigma' = GCN_{\sigma}(A,X)=A_{norm} ReLU(A_{norm} W_0 X) W_2\\
\end{aligned}
\end{equation}

In which $W_i$ is weight matrix, $i=0,1,2$. $A_{norm}$ is the symmetrically normalized adjacency matrix:\\
 \begin{equation}
\begin{aligned}
A_{norm} = D^{-\frac{1}{2}} AD^{-\frac{1}{2}}\\
\end{aligned}
\end{equation}

Note that first layer of $GCN_\mu(A,X)$ and $GCN_{\sigma}(A,X)$ share the same weight matrix with each other.
Then the distributions of latent variables are:
\begin{equation}
\label{encoder}
\begin{aligned}
q\left ( \varepsilon_i |A,X\right ) = N\left  (\mu_i , diag(\sigma_i)\right )\\
\end{aligned}
\end{equation}

(2) Decoder:\\
This process aims to regenerate the adjacency matrix $A$ in graph $\mathcal {G = (V, E)}$. Then decoder of adjacency matrix is given by an inner product between $\varepsilon _i$: \cite{b2}

 \begin{equation}
\begin{aligned}
\label{decoder}
p\left(A_{i j}=1| \varepsilon _i, \varepsilon _j\right)=\sigma\left(\varepsilon _i^{\top} \varepsilon _j\right)
\end{aligned}
\end{equation}
}

 

{\bf Linear Structured Causal Model (SCM).}
Linear SCM defines a causal system with linear equations representing the semantics as follows~\citep{shimizu2006linear,yang2021causalvae}, with independent exogenous factors $\epsilon \in \mathbb{R}^K$ and endogenous variables $z \in \mathbb{R}^K$,
\begin{equation}
    \begin{aligned}
        {z}={\Phi}^T {z}+{\epsilon}=\left(I-{\Phi}^T\right)^{-1} {\epsilon}, {\epsilon} \sim \mathcal{N}({0}, {I})
    \end{aligned}
\end{equation}
where $\Phi\in \mathbb{R}^{K \times K}$ is an adjacency matrix for a directed acyclic graph (DAG) that captures the causal structure of $n$ concepts. \textit{Note we use the same letters here to avoid the abuse of notation.}

{\bf Disentangled Causal Process (DCP).}
DCP studies disentanglement in the latent space by considering confounding variables, which results in theoretically sound properties as opposed to heuristics in the prior~\citep{suter2019robustly}.
Its detailed definition is as follows.
\begin{defi}
Given $m$ causal generative factors as ${G}=\left[G_1, \ldots, G_m\right]$ and $L$ confounders as $C =\{C_1, ... ,C_L\}$, 
causal disentanglement for the observation  $X$  is possible if and only if $X$ can be represented in a SCM context as follows,
\begin{equation}
\label{C}
C \leftarrow {\zeta}_C
\end{equation}
\begin{equation}
\label{G}
G_i \leftarrow q_i(H_i^C,\zeta_i),i=1,2,...,m
\end{equation}
\begin{equation}
\label{X}
X \leftarrow g(G, \zeta_x)
\end{equation}
Here,   
$H_i^C \in \{C_1, ... ,C_L\}$ is the father node of $G_i$, i.e.,
$H^C_i \rightarrow G_i$ holds regarding causality. ${\zeta}_c, \{\zeta_i\}_{i=1,2,...,m}, \zeta_x$ are independent noise variables. 
{Note $q_i$ and $g$ are predefined functions.}
\end{defi}


\eat{

{\bf Sub-optimality in VAEs.}
Note that real-world data often has built-in correlations, with its distribution written as $ p^{*}(X)=\int_{Z^*}p^{*}(X|{Z^*})p^{*}(Z^*)dZ^*$, where $p^{*}(Z^*)\neq \prod_{i}p(Z^*_i)$.
Typical VAEs will produce latent representation $Z$ and infer the distribution for $X$ as $ p_{\theta}(X)=\int_{Z}p_{\theta}(X|Z)p(Z)dZ$, where $p(Z)= \prod_{i}p(Z_i)$ and $\theta$ denotes the learnable parameters.
\citet{trauble2021disentangled} has shown that when equating $p_{\theta}(X)$  to the optimal $p^{*}(X)$, the process tends to obtain a sub-optimal solution.
Therefore, the VGAEs share the same problem and we will address this common problem in the following.
}

\section{Theory}

In this section, we first formulate the problem of causal disentangled representation learning in VGAE.
We next present a theoretical analysis of causal disentanglement, where we provide a tight upper bound to approximate the optimum (see Section~\ref{sec:approximation}).
After that,   we introduce a practical solution to accomplish this approximation together with its properties (see Section~\ref{sec:cc_dis}).





{\bf Problem Formulation.}
Denote the input graph data as $X$ and its optimal latent factor as $Z^*$~\citep{trauble2021disentangled}, the optimal data distribution is formulated as   $ p^{*}(X)=\int_{Z^*}p^{*}(X|{Z^*})p^{*}(Z^*)dZ^*$, along with an non-i.i.d. assumption on $Z^*$, i.e., $p^{*}(Z^*)\neq \prod_{i}p(Z^*_i)$.
As a common solution to disentanglement,  VGAE disentangles input data into latent representation $Z$, and
the corresponding data distribution is $ p_{\theta}(X)=\int_{Z}p_{\theta}(X|Z)p(Z)dZ$, where $\theta$ is the learnable parameters.
{
Given that all correlations can be modeled as causal structures~\citep{scholkopf2022statistical}, we let $Z$ possess a causal structure and define this structure with DCP.
Having no labels from $Z$,
the goal of unsupervised causal disentangled representation learning in VGAE is to achieve an optimal latent factor  $Z^*$  while making  $p_{\theta}(X) = p^{*}(X)$ always hold.
}


\subsection{ A theoretical analysis of causal disentanglement } 
\label{sec:approximation}

\begin{defi}
Given the the process~(\ref{C}) and~(\ref{G}) in the DCP, we obtain the generative factors as  $ G=Q({\zeta}^{+})$, with $\zeta_G=\{\zeta_1,\zeta_2,...,\zeta_m\}$ and $ \zeta^{+}={\zeta_C}\cup {\zeta_G}$.
The distribution of data  can be attained as $p_{\theta_{\zeta_x}}(X)=\int_{G}p_{\theta_{\zeta_x}}(X|G)p(G)dG$.
\end{defi}

This definition suggests that a given data can be represented with causal generative factors while having no assumption of independence as in the previous VAE. 
Based on this, as we will see in Theorem~\ref{P1}, the causal disentanglement guarantees an optimal solution to attain the following: 1) the distribution consistency between the input data and the predicted one, i.e., $p_{\theta}(X) = p^{*}(X)$, and 2) the optimal latent factor $Z^*$.
{In contrast, traditional VAE imposes an independence assumption on Z and suffers a sub-optimality solution~\citep{trauble2021disentangled} w.r.t the true data distribution, leading to $p_{\theta}(X) \neq p^{*}(X)$.}
With this difference, the above causal disentanglement {avoids such a assumption}.


\begin{theo}
\label{P1}
Given independent Normal distributed variables $N = {N_1, ..., N_K}$, there exists an optimal causality
modeling function $Q$ that represents the causal generative factor $G =  \{G_1, ..., G_K \}=Q(N)$, equating to the optimal disentangled latent factor $Z^*$, while holding $p_{\theta_{\zeta_x}}(X) = p^{*}(X)$.
\end{theo}

We defer the proof of Theorem~\ref{P1} to the Appendix~\ref{App1}.
Theorem~\ref{P1} proves that in a casual setting, there must \textbf{exist an optimal solution} for the VAE.
Next, we introduce a generalized causal generative factor expression that unifies the base for causal disentanglement.
\begin{defi}
\label{def:general_factor}
(general causal generative factor expression).
Given independent Normal distributed variables $ N=\{N_1,...,N_K\}$ and a 
matrix $A\in \mathbb{R}^{K \times K}$, 
any causal generative  factor can be formulated as $G_i=Q_i(B_i)$, 
where $B =  A * {diag(N)}$. 
\end{defi}

According to Theorem~\ref{P1} and the above definition, we attain a \textbf{unified optimal expression} of generative factor as follows (detailed proof is at Appendix~\ref{App2}):
\begin{prop}
\label{prop:tri_A}
(unified optimal generative factor expression).
Let $A$ as a lower triangular 
matrix, then the expression of optimal generative factors can be unified as 
$G_i=Q_i(B_i)$, 
where $B =  \hat{A} * {diag(N)}$ and $\hat{A}$ is permuted from $A$. 
\end{prop}


Having established the connections between optimal generative factors in Proposition~\ref{prop:tri_A}, we arrive at a necessary condition for optimal factors.
Whereas in this paper we aim to acquire both the necessary and sufficient conditions for the optimal factors.
Solving these two together yields an analytical solution for the optimal factors (at Appendix~\ref{App1}), making the implementation difficult in modern deep architectures.
Such a solution becomes infeasible alongside an unknown distribution for the optimal latent representations.
A practical solution is to \textbf{approximate the optimal factors}, within acceptable confidence, while being practically feasible.

Provided a representation base $B_i$, assume the existence of an approximated generative factor to the optimal one, $Q_i(B_i)$, over the same space, denoted as $Q^{'}_i(B_i)$. 
We derive a tight upper bound on the approximation error by setting $Q^{'}_i$ as a linear function, as shown in Theorem~\ref{theo_linear}.

\begin{theo}
\label{theo_linear}
Given $B_i$ in Proposition~\ref{prop:tri_A}, and 
 $N_i$ with an interval of $N_i \in \left [\mu_i-\delta,\mu_i+\delta \right]$, $i=1,2,...,K$, 
for an optimal $Q_i(B_i)$,
there exist a linear function ${Q^{'}_i}$ make $Q^{'}_i(B_i)$ , 
the absolute error has such bond:$|Q_i(B_i)-Q^{'}_i(B_i)|\le O_i( \delta)$, 
where
{$O_i(\delta)=a_{i}+\delta\Lambda_i$, $\Lambda_i=b_{i}+\sum_{t=1}^{i}c_{it}d_t^{\delta^2}$, $a_i,b_i,c_{it},d_t$ are constant unrelated to $\delta$, $0<d_t< 1$, $i=1,2,...,K$ and $t=1,2,...,K$, $\delta$ is a non-negative real number unrelated to distribution of $N$}
. 
\end{theo}

Please see Appendix~\ref{App3} for more details.
Theorem~\ref{theo_linear} suggests that over $95\%$ probability,  the range of $N_i$ is within $\left[\mu_i-2\sigma_i,\mu_i+2\sigma_i\right]$ and hence the error is bound by  $O_i( 2\sigma_i)$ , i.e., the  bound is nearly constant with $95\%$ confidence.
Note that we assume the optimal latent representation $(Z^*_1,Z^*_2,...,Z^*_K)$ as a linear uniform distribution.
One could arrive at different bounds with distributions, and we take the uniform distribution for simplicity.

\subsection{Concept-free Causal Disentanglement}
\label{sec:cc_dis}
Theorem~\ref{theo_linear} says we can obtain an approximated optimal generative factor by appointing the projection function linear.
This approximation enables a practical implementation toward the optimal latent factor.
Formally, we introduce the linear projection-based generative factor: 
\begin{prop}
\label{def:general_factor}
(Approximated generative factor expression).
Given independent Normal distributed variables $ N=\{N_1,...,N_K\}$, a lower triangular matrix $\Tilde{A}_i$, 
a causal generative  factor can be formulated as $G'_i=Q'_i(B_i)=\Tilde{A}_i*N'$, where $\Tilde{A} $ is obtained by  permuting a lower triangular matrix. 
\end{prop}
The proof is given in the Appendix~\ref{App4}.
We set our causal disentanglement in the context of the Structural Causal Model (SCM) and focus on a linear SCM because of its simplicity.
Following this, we formalize the causal structure in $G$ as follows,
\begin{equation}
~\label{eq:lscm_G}
    \begin{aligned}
        G'=\Phi G'+ \varepsilon \rightarrow G'= (I-\Phi)^{-1}\varepsilon^T
    \end{aligned}
\end{equation}
where $\Phi\in \mathbb{R}^{K \times K}$  is a {DAG adjacency matrix} and $\varepsilon$ is a independent variable.
The resulting $ (I-\Phi)^{-1}$ is also a {permuted} low triangular matrix, see the proof in Appendix~\ref{Proof3}. 
Note that Proposition~\ref{def:general_factor} is for a general causal setting. 
Letting $N= \bm{\varepsilon}$, the above linear representation based on SCM shares the same expression as that in the proposition, and thus inherits the ideal property of approximating the optimal latent factor $Z^*$. 
Furthermore, the two variables in Eq.~\ref{eq:lscm_G} are learned from data in a straightforward manner, without any labels for supervision.
Denoted each $G_i$ as a concept \citep{kumar2017variational}
we arrive at an unsupervised causal disentanglement that does not require any concept labels, called concept-free causal disentanglement.






{
In unsupervised disentanglement learning, along with the linear Gaussian assumption, the identifiability problem~\citep{locatello2019challenging} often arises due to the discrepancy between the pre-defined concepts and the learned ones.}
Without supervision, we  cannot achieve these pre-defined concepts, especially given limited data.
However, as we will see in the following theorem, these pre-defined concepts are attainable when sufficient data is accessed.
Since these concepts hold in multiple samples, making them the ground truth.




\begin{theo}
~\label{theo:truth}
Given $n$ observations $\{X^{(1)},X^{(2)},...,X^{(n)}\}$ sampled from the same distribution $p^*(X)$, along with their corresponding optimal generative factors  $\{Z^{(1)},Z^{(2)},...,Z^{(n)}\}$, 
the function of these generative factors will converge to the same ground truth (GT) concept.

\end{theo}

A formal version of
Theorem~\ref{theo:truth} and its proof can be found at Appendix~\ref{App7}. 
More importantly, we believe concepts obtained by the theorem are better than human-labeled concepts because these are limited and may involve bias.
The above discrepancy does not always imply errors in the learned concepts, and conversely, the latter can be a compensation for human-defined ones.

Besides,
Theorem~\ref{theo:truth}  enables guaranteed learning toward the ground truth (GT) concepts and leads to the following property:
\begin{Properties}
\label{prop:consistency1}
  \item (Consistency of generative factors). 
  Given observations sampled from the same distribution, each sample's optimal generative factors, i.e., $Z^*$, capture a portion of   GT concepts, 
  implying that one can approximate the GT concepts with a merging of $Z^*$, where we call the merged one an approximated concept.
\end{Properties}


%

The consistency property implies that concepts learned from individual samples capture the GT concepts shared by all data from the same distribution, making these concepts adaptable.
Therefore, under the same distribution, transferring concepts from observed data to newly sampled data benefits the learning of new data, thus significantly reducing the data demand and avoiding training from scratch.


\eat{
In unsupervised disentanglement learning, the identifiability problem~\citep{locatello2019challenging}  often arises due to the discrepancy between pre-defined concepts and learned ones.
We argue this may not be a significant concern because human-defined concepts are limited and may involve bias.
The above discrepancy does not always imply errors in the learned concepts, and conversely, the latter can be a compensation for human-defined ones, as Theorem~\ref{theo:truth} shows learned concepts will converge to the ground truth.
}


\eat{
The merits of concept-free causal disentanglement are at least two.
First, in contrast to supervised causal disentanglement, our concept-free solution does not require any labels for concepts but attains the latent ones directly from the data.
The concept-free disentanglement is of particular importance since well-defined concepts are often not accessible, and learnable latent concepts may be alternatives~\citep{higgins2016beta}.
Second, as we will see in Theorem 2, concept-free causal disentanglement offers globally parameterizable concepts that are closer to the truth.
}



\section{Method}
\label{sec5}


In this section, we propose a novel VGAE with a causal disentanglement model, namely Concept-free Causal VGAE (CCVGAE), whose goal is to obtain optimal disentangled latent representations.
We also introduce a concept-free causal disentanglement framework in a meta-learning setting, called concept-free causal Meta-Graph (CC-Meta-Graph), to harness the property of concept consistency.
We begin by introducing the definition of \ourmeth as follows,

\begin{defi}
(CCVGAE). 
Given an input graph's adjacency matrix $Adj$ and node attributes $Attri$,
the proposed CCVGAE is defined by:
\begin{itemize}
\item  A prior data distribution $p^{*}(Z^*)\neq \prod_{i}p(Z^*_i)$ roots on a set of causal structured latent factors $Z^*=\{Z^*_1, Z^*_2,...Z^*_K \}$.

\item An encoder  is composed of a GNN-based compression component and a causal disentanglement component.
The former employs GNN to compress the adjacency matrix $Adj$ and node attributes $Attri$ 
into a low-dimensional latent space as $\epsilon$.
The latter (parameterized by $\phi$) performs our concept-free causal disentanglement with $\epsilon$ as input, 
optimizes the underlying causal structure $\Phi$ in the learning procedure,
and outputs the posterior approximation parameters: $q_\phi({G'} \mid {\epsilon}, \Phi)$ ( see Eq.~\ref{eq:lscm_G}).

\item A decoder $p_\psi(Adj \mid G')$ that takes the obtained latent factor $Z$ to infer the adjacency matrix of the input graph data and is parameterized by $\psi$, i.e., $p\left(Adj_{i j}=1 \mid {G}_i, {G}_j\right)=\sigma\left({G}_i^{\top} {G}_j\right)$, where  $\sigma(\cdot)$ is the logistic sigmoid function.
\end{itemize}
\end{defi}



{\bf Optimization Objective.}
The optimization of \ourmeth is to encourage an equivalence between the approximated distribution $p_{\theta_{G}}(X)$ and the optimal one $p^{*}(X)$.
In particular, the evidence lower bound (ELBO) is used to minimize the divergence between the above two distributions, and to enforce that the distribution of  $\varepsilon$  is independent Gaussian, as follows,
\begin{equation}
\begin{aligned}
\footnotesize
&\mathcal{L}_G=E_{q(\hat{G}|Adj,Attri)}log(p(Adj|\hat{G}))+\\
& KL(q(\bm{\varepsilon}|(Adj,Attri))|N(0,I)).
\end{aligned}
\end{equation}
Apart from minimizing distribution divergences, we also want to shorten the distance between the observation and the recovered one by measuring the mean squared error (MSE), written as:
$\mathcal{L}_{MSE} = MSE(Attri,p(Attri|\hat{G}))$.
 
Meanwhile, performing causal structure modeling demands a DAG constraint on $\Phi$. For the convenience of optimization, we impose a differentiable constraint function~\citep{yu2019dag} as: $\mathcal{L}_{\Phi} = tr((I-\frac{r}{K}\Phi\circ \Phi)^K)-K$,
where $r$ is an arbitrary positive number, $tr(\cdot)$ denotes trace norm and $K$ denotes the number of concepts. 
Combining the above loss functions, we derive the overall loss function as follows,
 \begin{equation}
\begin{aligned}
\mathcal{L}=-\mathcal{L}_G+\alpha \mathcal{L}_{\Phi}+\beta \mathcal{L}_{MSE},
\end{aligned}
\end{equation}
where $\alpha$ and $\beta$ are hyper-parameters.
\textit{The overall algorithm is in Appendix~\ref{App6}.}

\subsection{Concept-free Causal disentanglement Meta-graph}

Meta-Graph~\citep{bose2019meta} deals with the few-shot link prediction task: it aims to predict links on target graphs ($\mathcal{G}_{T}$) with a model trained on a few source graphs ($\mathcal{G}_{S}$), where the source and target graphs are drawn from the same domain.
Denoted the distribution over graphs in the same domain as $p(\mathcal{G})$, the distributions of the source and target graphs follow the same, i.e., $\mathcal{G}_{S} \sim p(\mathcal{G})$ and $\mathcal{G}_{T} \sim p(\mathcal{G})$. 
To accomplish this task, we demand high-quality adaptation that transfers the information in the training data to newly arrived data.

According to Property 1, our concept-free causal disentanglement can provide fast adaptation and hence is well suited for a meta-learning setting.
Meta-Graph employs traditional VGAE to capture information to supply an initialization for training a subsequent link prediction model.
Thanks to the consistency property, our proposed disentanglement solution can capture information (i.e., concepts) that is adaptable to newly arrived data.
To this end, we replace VGAE with \ourmeth and let the other components remain in the Meta-Graph, called CC-Meta-Graph. 
\textit{We present the corresponding algorithm in Appendix~\ref{App6}.}



	\section{Experiments}

\eat{
We conduct experiments in this section to test and understand the effectiveness of concept-free causal VGAE.
Specifically, we aim to answer the following research questions under traditional VGAEs' setting and meta-learning setting:\\
\todo{REWRITE following questions}
}
\eat{
Q1) How does our model perform when compared with state-of-the-art link prediction baselines in real world data sets?\\
Q2) Does DAG constrain (\ref{DAG constrain}) improve the accuracy of link prediction?\\ 
Q3) How does our model perform when data sets has correlation (even causality) as described in (\ref{p*}) compared with the state-of-the-art.\\
Q4) Does our model get a representation which is more suitable to transfer than state-of-the-art baselines, for example, the representation learned in some graphs can provide better initialization representation for training in other graphs?
}


\eat{
{\bf Datasets.}
We  experiment on $6$ graph benchmark datasets from various domains~\citep{pei2020geom,tang2009social}, including Cora, dRisk, Actor, Corn, Texas, and Wisconsin.\textcolor{blue}{Cora and Actor have no citation?}
\todo{citation for cora....}
Table~\ref{T1} presents the statistics of these datasets, including the numbers of nodes, edges, and node attributes. 

\begin{table}[ht]
\centering
\label{T1}
\caption{ Statistics of datasets in our experiments.
N, E,  and Attr denote the number of “Nodes”, “Edges”, and “Node Attributes” , respectively.
Note the initial number of edges for the synthetic data is $4894$. }
\begin{tabular}{c|ccc}
\hline
Dataset   & N &  E &  Attr  \\ \hline
Cora      & 2708  & 5429  & 1433     \\
Corn      & 183   & 295   & 1703     \\
Texas     & 183   & 309   & 1703     \\
Wisconsin & 251   & 499   & 1703     \\
dRisk     & 100   & 478   & 4        \\
Actor     & 7600  & 33544 & 931      \\ 
Synthetic     & 100  & 4984^*  & 16      \\ \hline
\end{tabular}
\end{table}

Note that dRisk is a data set transformed from dRiskKB~\citep{xu2014driskkb}, constructed from the biological text.
dRiskKB contains $12981$ nodes representing disease names, with weighted edges indicating correlations between disease pairs.
To simplify the dataset, we randomly select 100 nodes from dRiskKB and transfer the weighted edges to the non-weighted edges.
The dRiskKB does not provide node attributes, so we randomly generate $4$ dimensions of one-hot features as node attributes.

\eat{
Note that dRisk is a data set transformed from dRiskKB \citep{xu2014driskkb}, which is constructed from biological text. In dRiskKB, there are 12981 nodes representing disease names, and weighted edges between nodes indicating the correlation between disease pairs. To simplify data set, we randomly select 100 nodes from dRiskKB and transfer weighted edges to non-weight edges. The dRiskKB does not provide the features of nodes, so we randomly generate 4 dimensions one-hot features as node features.
}

Considering that real-world datasets often have unknown causality, we thus construct synthetic data with controllable causality to answer the Q3.
In particular, we produce attributes of nodes, $X$, and the adjacency matrix, $A\in  \mathbb{R}^{100\times 100}$, as follows: $A=\sigma(Z\cdot Z^T)$ and $X=20Sin(Z)$. 
Here, we produce $Z$ using linear SCM to ensure its causality, Mathematically, we derive $Z = C^T Z + \varepsilon \rightarrow Z=(I-C^T)^{-1}\varepsilon$, where $C\in  \mathbb{R}^{16\times 16}$ is a random lower triangular matrix and $\varepsilon \in \mathbb{R}^{16}$ is an independent random vector with same variance normal distribution.
} 
 \eat{
 \begin{equation}
\begin{aligned}
A=sigmoid(Z\cdot Z^T)
\end{aligned}
\end{equation}
 \begin{equation}
\begin{aligned}
X=Sin(Z) \cdot 20
\end{aligned}
\end{equation}

 \begin{equation}
\begin{aligned}
\label{generate Z}
Z = C^T Z + \varepsilon \rightarrow Z=(I-C^T)^{-1}\varepsilon
\end{aligned}
\end{equation}

We set $C$ as $16 \times 16$ dimensional random lower triangular matrix. $\varepsilon$ is a 16 dimensional independent random vector with the same variance normal distribution. Set number of nodes as 100.
}

\eat{
{\bf Metrics.}
To evaluate our method, we perform the link prediction task and thus take two commonly used metrics in this area~\citep{kipf2016variational}: Area Under ROC Curve (AUC) and Average Precision (AP) scores.
All the experiment results are averaged over $3$ random seeds.
} 
\subsection{Task $1$: Link Prediction}

This experiment aims to study how the proposed method, \ourmeth, performs on the link prediction task when compared to state-of-the-art methods.

{\bf Datasets.}
We  experiment on $6$ graph benchmark datasets from various domains~\citep{sen2008collective, pei2020geom,tang2009social}, including Cora, dRisk, Actor, Corn, Texas, and Wisconsin.
Table~\ref{T1} presents the statistics of these datasets, including the numbers of nodes, edges, and node attributes. 

\begin{table}[ht]
\centering
\label{T1}
\caption{ Statistics of datasets in our experiments. 
{Note the initial number of edges for the synthetic data is $4894$.} $\# $ demotes number of. }
\begin{tabular}{c|ccc}
\hline
Dataset   &  $\# $ Node &  $\# $ Edge & $\# $ Attr  \\ \hline
Cora      & 2708  & 5429  & 1433     \\
Corn      & 183   & 295   & 1703     \\
Texas     & 183   & 309   & 1703     \\
Wisconsin & 251   & 499   & 1703     \\
dRisk     & 100   & 478   & 4        \\
Actor     & 7600  & 33544 & 931      \\ 
Synthetic     & 100  & $4984^*$  & 16      \\ \hline
\end{tabular}
\end{table}

\begin{table*}[!th]
\centering
\caption{
AUC ($\%$) and AP ($\%$)  scores for all  baselines on real-world datasets. 
Note that X-DGAE shows the best results among  all variations of DGAE, including 6-DGAE$_{\alpha}^{\beta}$, 36-DGAE$_{\alpha}^{\beta}$, and 64-DGAE$_{\alpha}^{\beta}$.  $*$ denotes results from the original article.} 
\footnotesize
 \scalebox{0.9}{
   \begin{tabular}{l|cccccccccc}
    \toprule
      \multirow{2}{*}{} &
      \multicolumn{2}{c}{GVAE} &
      \multicolumn{2}{c}{SIG-VAE} &
      \multicolumn{2}{c}{X-DGAE} &
      \multicolumn{2}{c}{\ourmeth} &
      \multicolumn{2}{c}{ \ourmeth w/o CC}
      \\
      ~ & 
      {AUC} & {AP} & {AUC} & {AP} & {AUC} & {AP} & {AUC} & {AP} & {AUC} & {AP}\\
      \midrule
        Cora 
        & 0.91${_{\pm0.02}}$ & 0.92${_{\pm0.01}}$ &  0.92 ${_{\pm0.01}}$ & {\bf 0.93 ${_{\pm0.02}}$} & {\bf0.93${_{\pm 0.02}}$} & 0.92${_{\pm 0.02}}$ & 0.85${_{\pm0.03}}$ & 0.85${_{\pm0.05}}$ & 0.72${_{\pm0.04}}$ & 0.73${_{\pm0.03}}$\\
        Corn 
        & 0.53${_{\pm0.03}}$ & 0.66${_{\pm0.06}}$   & 0.62${_{\pm0.05}}$ & 0.64${_{\pm0.03}}$ & 0.73${_{\pm 0.10}}$ & 0.77${_{\pm 0.10}}$ & {\bf 0.74 ${_{\pm0.06}}$} &{\bf 0.78 ${_{\pm0.04}}$} & 0.68${_{\pm0.06}}$ & 0.73${_{\pm0.05}}$\\
        Texas     & 0.51${_{\pm0.06}}$ & 0.59${_{\pm0.04}}$ & 0.60${_{\pm0.03}}$ & 0.63${_{\pm0.05}}$ & 0.46${^*_{\pm 0.09}}$ & 0.61${^*_{\pm 0.08}}$ & {\bf 0.75}${_{\pm0.07}}$ & {\bf 0.80}${_{\pm0.07}}$ & 0.74${_{\pm0.05}}$ & 0.75${_{\pm0.06}}$ \\
        Wisconsin  & 0.57${_{\pm0.04}}$ & 0.68${_{\pm0.04}}$  & 0.68${_{\pm0.05}}$ & 0.69${_{\pm0.06}}$ & 0.54${^*_{\pm 0.09}}$ & 0.67${^*_{\pm 0.09}}$ & {\bf 0.75}${_{\pm0.04}}$ & {\bf 0.79}${_{\pm0.05}}$  & 0.68${_{\pm0.04}}$ & 0.69${_{\pm0.04}}$ \\
        dRisk & 0.61${_{\pm0.03}}$ & 0.62${_{\pm0.05}}$  & 0.58${_{\pm0.03}}$ & 0.56${_{\pm0.04}}$ & 0.73${_{\pm 0.11}}$ & 0.72${_{\pm 0.10}}$  & {\bf 0.75}${_{\pm0.06}}$ & {\bf 0.72}${_{\pm0.05}}$ & 0.63${_{\pm0.05}}$ & 0.62${_{\pm0.06}}$\\ 
        Actor      & 0.76${_{\pm0.07}}$ &0.81${_{\pm0.06}}$ & 0.77${_{\pm0.03}}$ & 0.80${_{\pm0.05}}$ & 0.77${_{\pm 0.02}}$ & 0.80${_{\pm 0.03}}$ & {\bf 0.78${_{\pm0.07}}$} & {\bf 0.81}${_{\pm0.06}}$ & 0.72${_{\pm0.03}}$& 0.76${_{\pm0.04}}$ \\
         \bottomrule
   \end{tabular}
     }
   \label{Table_main_comparison_1}
\end{table*}

\begin{table*}[!th]
\centering
\caption{
The performance of Meta-Graph-based baselines under different settings: varying number of meta-training loops and the requirement of meta-training data.
}  
   \label{Table_main_comparison_2}
\centering
\footnotesize
\scalebox{0.8}{
\begin{tabular}{c|cccccc|cccccc}
\toprule
                     & \multicolumn{6}{c|}{PPI}                                                                                                    & \multicolumn{6}{c}{FIRSTMM\_DB}                                                                                             \\
\multirow{2}{*}{loops} & \multicolumn{2}{c}{CC-Meta-Graph}       & \multicolumn{2}{c}{Meta-Graph}         & \multicolumn{2}{c|}{Rand-Meta-Graph}             & \multicolumn{2}{c}{CC-Meta-Graph}       & \multicolumn{2}{c}{Meta-Graph}         & \multicolumn{2}{c}{Rand-Meta-Graph}              \\

                             & $5\% $               & $10\%$               & $5\%$                & $10\%$               & $5\%$                & $10\%$               & $5\%$                & $10\%$               & $5\%$                & $10\%$               & $5\%$                & $10\%$               \\ 
                             \midrule
                             
10                           & {\bf 0.70}${_{\pm0.01}}$ & {\bf 0.76}${_{\pm0.01}}$ & 0.59${_{\pm0.02}}$ & 0.70${_{\pm0.01}}$ & 0.50${_{\pm0.01}}$ & 0.50${_{\pm0.00}}$ & {\bf 0.59}${_{\pm0.02}}$ & {\bf 0.61}${_{\pm0.01}}$ & 0.57${_{\pm0.01}}$ & 0.59${_{\pm0.01}}$ & 0.50${_{\pm0.01}}$ & 0.50${_{\pm0.01}}$ \\
30                           & {\bf 0.70}${_{\pm0.01}}$ & {\bf 0.77}${_{\pm0.01}}$ & 0.66${_{\pm0.01}}$ & 0.75${_{\pm0.02}}$ & 0.51${_{\pm0.00}}$ & 0.52${_{\pm0.01}}$ & {\bf 0.59}${_{\pm0.00}}$ & {\bf 0.61}${_{\pm0.00}}$ & 0.58${_{\pm0.01}}$ & 0.60${_{\pm0.00}}$ & 0.52${_{\pm0.01}}$ & 0.51${_{\pm0.00}}$ \\
50                           & {\bf 0.72}${_{\pm0.02}}$ & {\bf 0.77}${_{\pm0.00}}$ & 0.70${_{\pm0.01}}$ & 0.77${_{\pm0.01}}$ & 0.51${_{\pm0.01}}$ & 0.52${_{\pm0.01}}$ & {\bf 0.60}${_{\pm0.00}}$ & {\bf 0.62}${_{\pm0.02}}$ & 0.59${_{\pm0.01}}$ & 0.61${_{\pm0.01}}$ & 0.51${_{\pm0.02}}$ & 0.51${_{\pm0.00}}$ \\
70                           & {\bf 0.73}${_{\pm0.01}}$ & {\bf 0.77}${_{\pm0.00}}$ & 0.72${_{\pm0.01}}$ & 0.77${_{\pm0.01}}$ & 0.51${_{\pm0.00}}$ & 0.51${_{\pm0.00}}$ &{\bf  0.61}${_{\pm0.01}}$ & {\bf 0.62}${_{\pm0.01}}$ & 0.59${_{\pm0.01}}$ & 0.62${_{\pm0.00}}$ & 0.51${_{\pm0.00}}$ & 0.52${_{\pm0.01}}$ \\ 
\bottomrule
\end{tabular}
}
\end{table*}

Note that dRisk is a data set transformed from dRiskKB~\citep{xu2014driskkb}, constructed from the biological text.
dRiskKB contains $12981$ nodes representing disease names, with weighted edges indicating correlations between disease pairs.
To simplify the dataset, we randomly select $100$ nodes from dRiskKB and transfer the weighted edges to the non-weighted edges.
The dRiskKB does not provide node attributes, so we randomly generate $4$ dimensions of one-hot features as node attributes.

\eat{
Note that dRisk is a data set transformed from dRiskKB \citep{xu2014driskkb}, which is constructed from biological text. In dRiskKB, there are 12981 nodes representing disease names, and weighted edges between nodes indicating the correlation between disease pairs. To simplify data set, we randomly select 100 nodes from dRiskKB and transfer weighted edges to non-weight edges. The dRiskKB does not provide the features of nodes, so we randomly generate 4 dimensions one-hot features as node features.
}

Considering that real-world datasets often have unknown causality, we thus construct synthetic data with controllable causality.
In particular, we produce attributes of nodes, $X$, and the adjacency matrix, $Adj\in  \mathbb{R}^{100\times 100}$, as follows: $Adj=\sigma(Z\cdot Z^T)$ and $Attri=20Sin(Z)$. 
Here, we produce $Z$ using linear SCM to ensure its causality, Mathematically, we derive $Z = C^T Z + \varepsilon \rightarrow Z=(I-C^T)^{-1}\varepsilon$, where $C\in  \mathbb{R}^{16\times 16}$ is a random lower triangular matrix and $\varepsilon \in \mathbb{R}^{16}$ is an independent random vector with same variance normal distribution.

{\bf Baselines.}
We compare  \ourmeth to three prior methods: 
(1) VGAE~\citep{kipf2016variational}, which is the first graph-based VAEs;
(2) SIG-VAE~\citep{hasanzadeh2019semi}, which uses  a hierarchical variational framework for encoder and a Bernoulli-Poisson link decoder;
(3) DGAE~\citep{wu2022stabilizing} incorporates standard auto-encoders (AEs) into GAEs to enhance the ability of modeling structured information.

{\bf Metrics.}
To evaluate our method, we perform the link prediction task and thus take two commonly used metrics in this area~\citep{kipf2016variational}: Area Under ROC Curve (AUC) and Average Precision (AP) scores.
All the experiment results are averaged over $3$ seeds.

{\bf Implementation details.}
We train the proposed model for $200$ iterations using Adam.
As for the mean and variance, we use $32$-dimensional and $16$-dimensional GCN layers to implement, respectively.

{\bf Main results.} 
We benchmark all the methods across $6$ real-world datasets.
In Table~\ref{Table_main_comparison_1}, we observe that \ourmeth (ours) can reliably compete others with up to $29\%$ improvement regarding AUC and $19\%$ improvement regarding AP.
Recall that \ourmeth improves on VGAE by integrating a causal layer to encourage disentangled representations, suggesting that the significant improvement is due to the expressiveness of those disentangled representations. 
SIG-VAE improve the representation by imposing graph structure-aware distributions instead of independent Gaussian, which results in better performance than VGAE. 
DGAE enhances VGAE by deepening GCN layers resulting in a better result than VGAE, especially for non-Euclidean data.

Additionally, we find that the performance on the Cora dataset shows different trends than other datasets.
We hypothesize that such data could be generated under nearly independent factors, thus countering the validity of our assumption, i.e., $ p^{*}(X)=\int_{Z^*}p^{*}(X|{Z^*})p^{*}(Z^*)dZ^*$ with $p^{*}(Z^*)\neq \prod_{i}p(Z^*_i)$, and resulting in poor performance.


We also experiment on the synthetic data with a predefined causal structure and achieve advantages as before.
In Figure~\ref{fig:synthetic AUC}, we present performance for all methods by  varying the variance of $\varepsilon$ in a large range: from $10$ to $300$. 
Interestingly, we find that the performance varies little as the noise level increases, implying that these VAE-based methods are robust to noise as they capture the variance of the distribution well.
Together, the robustness of our model benefits from the modeling of causality and variances.

\eat{
\begin{table*}[ht]
\label{T3}
\caption{The AUC of different inner loops with baseline framework meta-graph (encoded by VGAE), the proposed model, and the framework which is encoded randomly.\todo{replot}}
\centering
\begin{tabular*}{0.7\linewidth}{c|ccc|ccc}
\hline
dataset     & \multicolumn{3}{c|}{PPI}                  & \multicolumn{3}{c}{FIRSTMM\_DB}           \\ \hline
inner\_loop & PM(Ours) & VGAE & Random      & PM(Ours) & VGAE & Random      \\ \hline
10          & \textbf{0.70}          & 0.59        & 0.50   & \textbf{0.65}          & 0.57        & 0.50   \\
30          & \textbf{0.70}          & 0.66        & 0.51   & \textbf{0.65}          & 0.58        & 0.52   \\
50          & \textbf{0.72}          & 0.70        & 0.51   & \textbf{0.65}          & 0.59        & 0.51   \\
70          & \textbf{0.73}          & 0.72        & 0.51   & \textbf{0.65}          & 0.59        & 0.51   \\  \hline
\end{tabular*}
\end{table*}
}

\eat{
\begin{table*}[]
\centering
\begin{tabular*}{0.7\linewidth}{c|ccc|ccc}
\hline
dataset     & \multicolumn{3}{c|}{PPI}             & \multicolumn{3}{c}{FIRSTMM\_DB}      \\ \hline
inner\_loop & PM(Ours) & VGAE & Random & PM(Ours) & VGAE & Random \\ \hline
10          & 0.76          & 0.70        &        &               & 0.59        &        \\
30          & 0.77          & 0.75        &        &               & 0.60        &        \\
50          & 0.77          & 0.77        &        &               & 0.61        &        \\
70          & 0.77          & 0.77        &        &               & 0.62        &        \\ \hline
\end{tabular*}
\end{table*}
}

\eat{
With the increase of variance of $\varepsilon$ in equation(\ref{generate Z}) from 10 to 500, the AUC of VGAE , SIG-VAE and \ourmeth in synthetic dataset are as Figure \ref{fig:synthetic AUC}.
}



\subsection{Task $2$: Few Shot Link Prediction}


In this experiment, we aim to demonstrate the effectiveness of the proposed CC-Meta-Graph.
As this is a meta-learning model, it consists of a meta-training phase followed by a testing phase, and its goal is to transfer knowledge from meta-training to the test phase.
We will investigate the performance of all methods regarding (1) the number of meta-training loops and (2) the meta-training data requirement because these are the keys to a meta-learning model's performance.



\eat{
To answer Q4, 
we have experiment on the task of few shot link prediction. This task is to predict missing edges across multiple graphs using only a small sample of known edges. 
The graph is divided into several sub-graphs, some of them are training data sets and others are testing ones. Before test data sets, we use representation in training data sets as initialization.

To figure out how is the initialization representation in Q4, we can observe model performance when inner-loop number K is small. We can infer that, the less the inner-loop number K (how many times to train each sub-graph) is, the model will be more affected by the initial parameters. Therefore, when inner-loop number K is small, the better performance can be owe to better initialization learned by other sub-graphs.
}

\textbf{Baselines.} 
Our experiment consists of three baselines corresponding to Meta-Graph~\citep{bose2019meta} modifications, which employ pre-trained VGAEs, pre-trained CC-VGAEs, and randomness for initialization, called Meta-Graph, CC-Meta-Graph and Rand-Meta-Graph, respectively. 
In particular, the first baselines two are pre-trained on training graphs 
and fine-tuned on test graphs.

\eat{
In meta-graph, VGAE was adopted as inference model. To compare the proposed model in few shot, we replace inference model with the proposed model named as causal-meta-graph. 
}

\begin{figure*} [!t] 
    \centering
 \subfigure [] 
    {
        \includegraphics[width=0.3\textwidth]{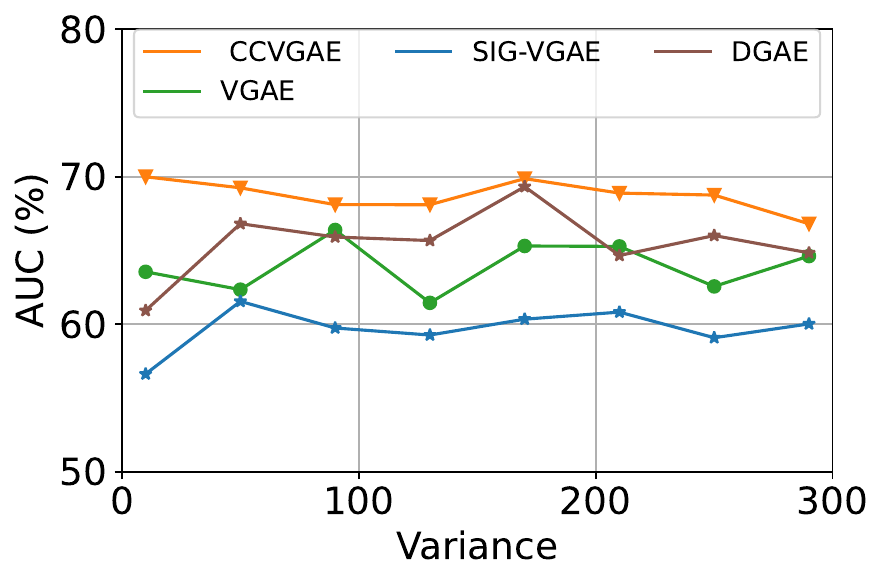}
         \label{fig:synthetic AUC}
    }
    \hspace{-0.1in}
        \subfigure []
    {
        \includegraphics[width=0.32\textwidth]{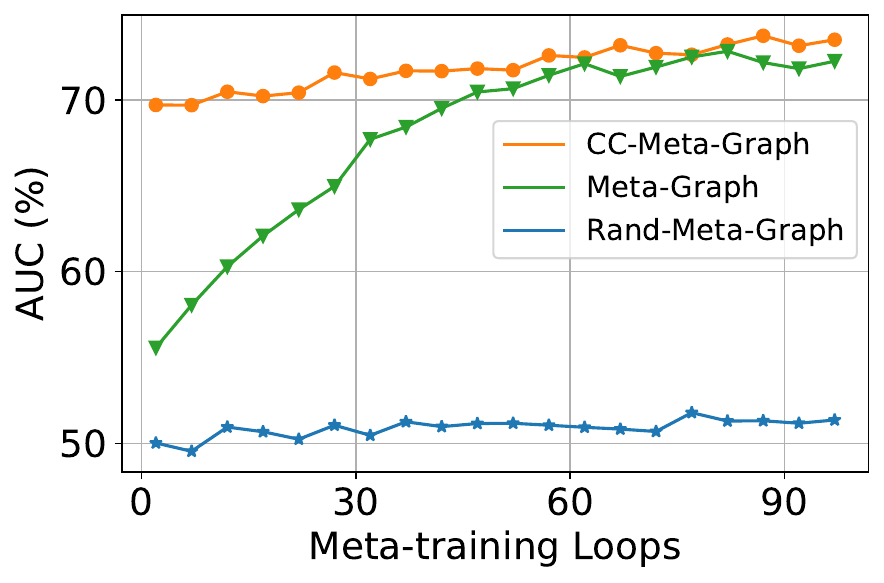}
         \label{fig:my_label}
    }
        \hspace{-0.1in}
        \subfigure []
    {
        \includegraphics[width=0.3\textwidth]{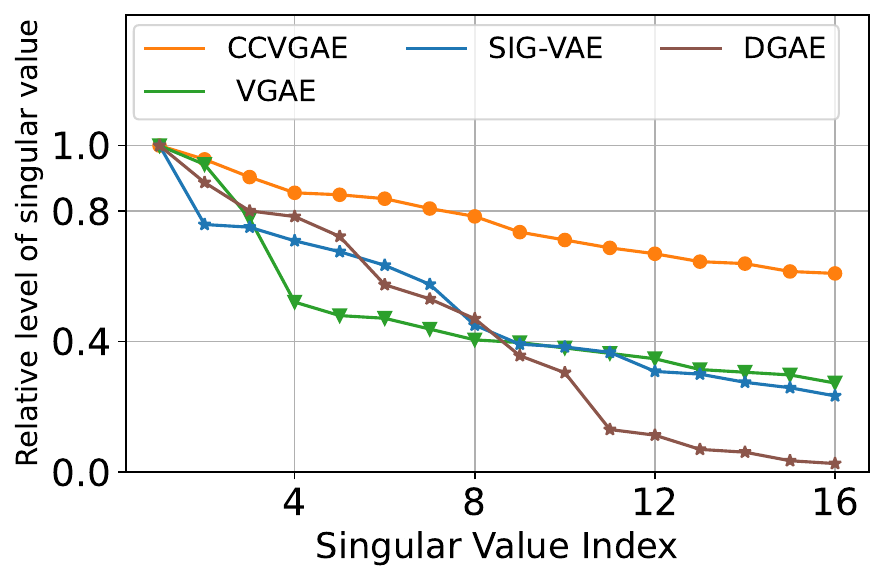}
         \label{fig:svd_disentangle}
    }
    \caption{(a): The comparison of all baselines of the few shot link prediction task on the synthetic data set. 
    The  x-axis denotes the variance of $\varepsilon$, which is used to construct the synthetic dataset. 
 The index of the maximum is $1$, the smaller the value and the larger the index. (b):  The performance of three methods when varying the number of meta-training loops (the PPI dataset).
 (c): The redundancy reduction analysis for causal disentangled representation. 
 We take SVD of representations and normalize the eigenvalues to make the maximum as $1$. The X-axis is the index of sorted normalized singular value, i.e., the first one denotes the largest value.
 }
\end{figure*}

{\bf Datasets.}
We experiment on two benchmark datasets~\citep{bose2019meta,zitnik2017predicting}, including  protein-protein interaction (PPI) and FirstMM DB. 
In this experiment, for all datasets, we perform link prediction by meta-training on a small subset of edges and then infer unseen edges.
Under all settings, we use $80\%$ of these graphs to pre-train weights and $10\%$ as meta-validation, optimizing the global model parameters, and the rest for meta-testing.
In terms of link prediction, we train all methods with two different settings: $5\%$ and $10\%$ edges of graphs, trying to see the effectiveness of using the data.
Apart from meta-training, we always use $20\%$ of edges for validation and the rest for testing. 





\eat{
Note that we only update causal layer after sub-graph training, because we can assume that the locally applicable causality may not be globally applicable, but the globally applicable causality must be locally applicable.
}


{\bf Main results.} 
In Table~\ref{Table_main_comparison_2}, we present the performance of all methods under different settings.
Our method, CC-Meta-Graph, outperforms others consistently, providing up to a $11\%$ absolute improvement.
Notably, we can see that with $5\%$ of the data, the performance of CC-Meta-Graph is competitive with the others given $10\%$, suggesting that our model can produce better generalizable representations with much less data and align with the consistency property.

We also evaluate how our model behaves under different meta-training epochs, as shown in Figure~\ref{fig:my_label}.
Our method shows near-optimal performance even with only a few loops as opposed to a few dozen loops for {Meta-Graph}.
Since Rand-Meta-Graph passes random values to the fine-tuning stage and thus can not benefit from the meta-training mechanism, resulting in the worst performance consistently.

To summarize, the superiority of our model validates the effectiveness of transferring global information to newly arrived data, even with significantly small data and only a few training loops, making our proposed method applicable under a limited budget.

\eat{
With the change of training loops in test dataset, the proposed model has a stably higher AUC than baseline. 

When the training loops is very small, the proposed model has higher AUC, which means that the proposed model can provide better initialization parameters in link prediction.
}

\subsection{Ablation Study}
{\bf Module Importance.} Recall that, for representation learning, we employ a causal structure to enforce disentanglement, which is DAG-structured $\Phi$ in Eq.~\ref{eq:lscm_G}.
Thereby, we investigate how our method performs without such a causality structure constraint, called \ourmeth w/o CC.
In Table~\ref{Table_main_comparison_1}, we find that the model without the DAG constraint , i.e., \ourmeth w/o CC, reduces the absolute performance by $12\%$  and $12\%$ regarding AUC and AP, respectively.
These ablation results suggest the necessity of causal structure in our model.

{\bf {The necessarily of $\mathcal{L}_{MSE}$}.} 
We investigate how our method performs without $\mathcal{L}_{MSE}$, called \ourmeth w/o MSE. 
In table~\ref{ablation_MSE}, we find \ourmeth and \ourmeth w/o MSE have similar performance (within 2$\%$ absolute gap) in Corn, Texas, Actor. In Cora, dRisk, Wisconsin, \ourmeth w/o MSE reduces the absolute performance by up to $5\%$  and $7\%$ regarding AUC and AP, respectively. These results suggest that $\mathcal{L}_{MSE}$ may slightly improve performance in some data, but not major.

\eat{
\begin{figure}
    \centering
    \includegraphics[width=0.4\textwidth]{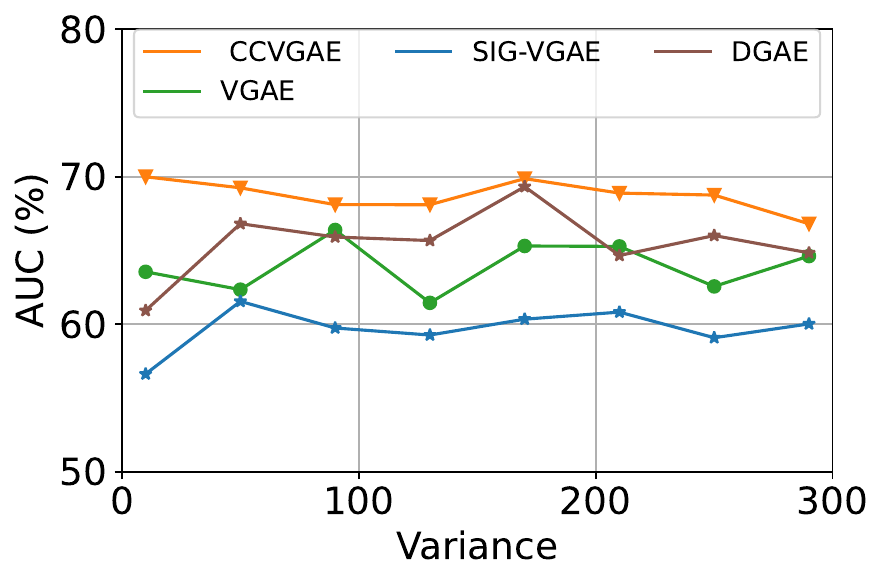}
    \caption{ 
    The comparison of all baselines of the few shot link prediction task on the synthetic data set. 
    The  x-axis denotes the variance of $\varepsilon$, which is used to construct the synthetic dataset. 
    }
    \label{fig:synthetic AUC}
\end{figure}

\begin{figure}
    \centering
    \includegraphics[width=0.45\textwidth]{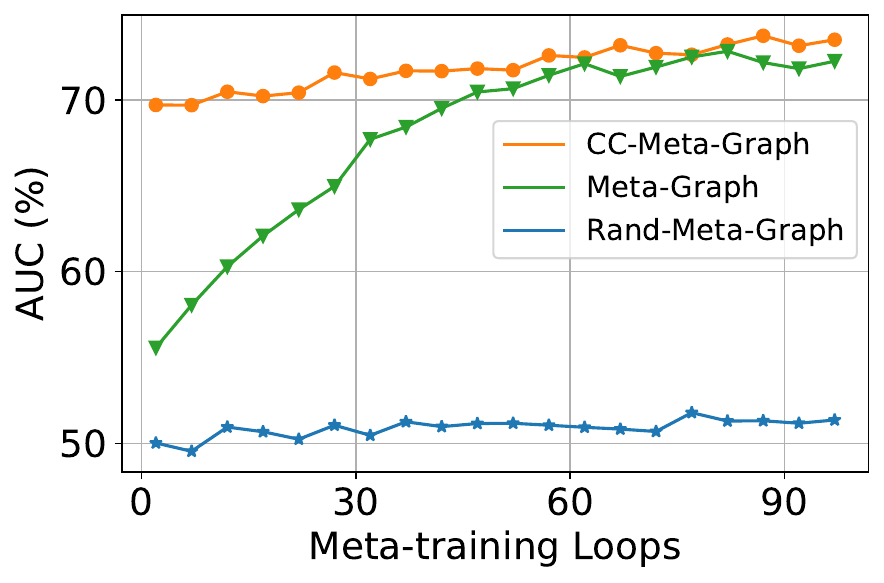}
    \caption{The performance of three methods when varying the number of meta-training loops (the PPI dataset).    
    }
    \label{fig:my_label}
\end{figure}
\begin{figure} 
    \centering
    \includegraphics[width=0.4\textwidth]{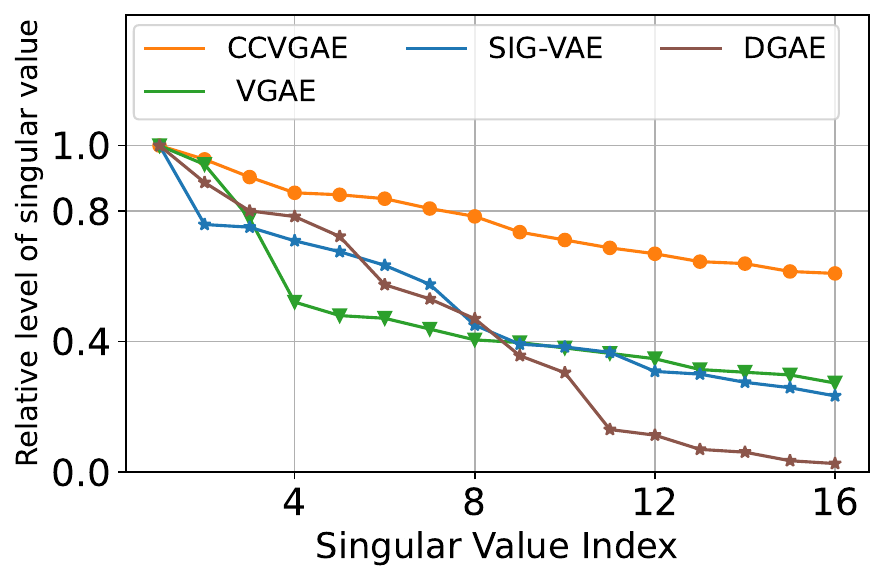}
    \caption{ 
    The redundancy reduction analysis for causal disentangled representation. 
 We take SVD of representations and normalize the eigenvalues to make the maximum as $1$. The x-axis is the index of sorted normalized singular value.
 The index of the maximum is $1$, the smaller the value and the larger the index.
}
    \label{fig:svd_disentangle}
\end{figure}
}

     
    
    

\subsection{Analysis on the redundancy reduction}
We now present a redundancy reduction perspective to understand the effectiveness of our disentangled representations.
In particular, we apply the singular value decomposition (SVD) on the obtained representations from Texas dataset and compare the magnitudes of their eigenvalues, i.e., the importance of each eigenvector.
In Figure~\ref{fig:svd_disentangle}, we observe that the singular values of our method decrease slower, demonstrating that the importance of these eigenvectors is less concentrated.
This implies that our representations are less redundant, making them more expressive under low-dimensional settings.

\eat{
A fundamental property of disentanglement is the independence of the components of the latent variables. Yet this seems hard to visualize.
\textcolor{blue}{I can understand here we need a motivation. But I think it is a little risky to say like that. Yes, some articles in disentanglement paid attention to independence, we can cite them here to support this view. But these articles are based on such assumption that generate factors are independent. However, we do not believe it. It may lead to some questions here. I think it is a evaluation about information redundancy, which is a property about representation learning instead of disentanglement. Do you know some articles using SVD to make analysis? What did they say?}
We instead show orthogonality in representations since independent variables have to be uncorrelated, and uncorrelated and orthogonal are nearly interchangeable.\textcolor{blue}{Too risky!  "uncorrelated and orthogonal are nearly interchangeable", they are two different concepts and two different definitions, one is about vectors, the other is about random variables.}
In particular, we apply the singular value decomposition (SVD) on the obtained representations and compare the magnitudes of their eigenvalues, i.e., the strength of orthogonality.
In Figure~\ref{fig:svd_disentangle}, we observe that the singular values of our method are much larger than the others, and the baselines with similar performance share similar eigenvalues, aligned with our assumption that orthogonality promotes better performance.\textcolor{blue}{Are there other articles proved or tried to prove this assumption? This assumption is a little strong. It is too hard to make such conclusion just by one experiment on one data set. It will be better if there are citations to support.}
}

\eat{
\begin{figure}[!th]
    \centering
    \includegraphics[width=0.4\textwidth]{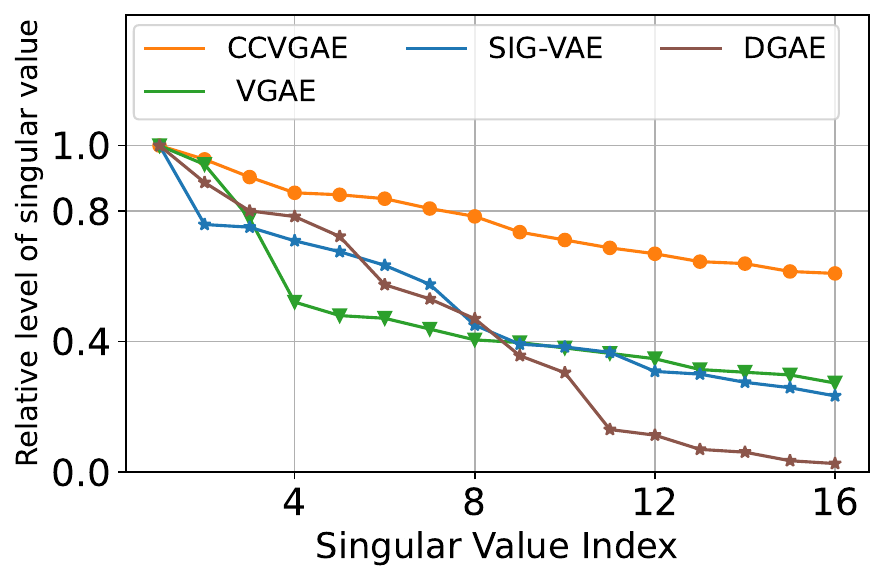}
    \caption{ \footnotesize
    The redundancy reduction analysis for causal disentangled representation. 
 We take SVD of representations and normalize the eigenvalues to make the maximum as $1$. The x-axis is the index of sorted normalized singular value.
 The index of the maximum is $1$, the smaller the value and the larger the index.
}
    \label{fig:svd_disentangle}
\end{figure}
}

\eat{
\begin{figure}[!th]
    \centering
    \includegraphics[width=0.4\textwidth]{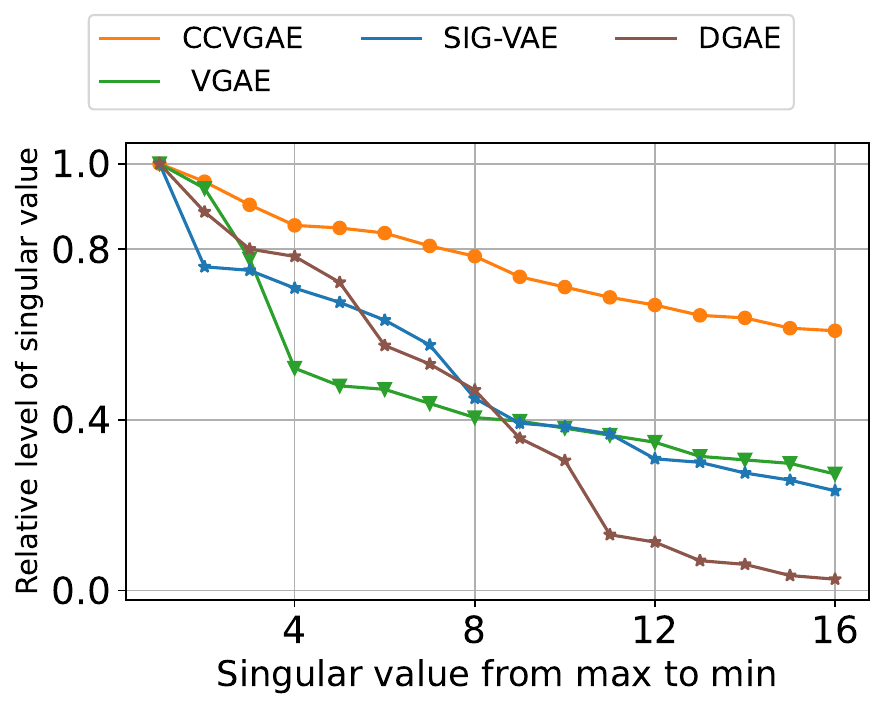}
    \caption{
    The redundancy reduction analysis for causal disentangled representation. 
 We take SVD of representations and normalize the eigenvalues to make the maximum as $1$. The horizontal axis is the index of sorted normalized singular value (the max value's index is 1, the less the value, the larger the index). 
}
\end{figure}
}
	\section{Conclusion}

In this paper, we provide a tight upper bound for the approximation of the optimal solution in the VAE framework, together with a practical solution, called Concept-free Causal Disentanglement.
We then propose an enhanced VGAE by a new causal disentanglement layer with the above idea, called \ourmeth. 
In addition, we discover the consistency of our derived concepts, which motivates us to develop a meta-learning model, called CC-Meta-Graph, aiming to transfer global information from limited data to new ones.
Our experimental results show the effectiveness of both models in the link prediction task and the few-shot one.

\eat{
(1) Our contribution is that we firstly implement causal disentanglement into graph representation learning, and the proposed model achieve competitive results in the dataset.
(2) The representation learned by the proposed model is more suitable to transfer, which can be applied when edges in graph are rarely known to improve the quality of initial parameters. In a small number of training loops, the proposed model has excellent performance in the link prediction.
}

	\nocite{langley00}
	
	\newpage
	\bibliography{reference}
	\bibliographystyle{icml2023}

	\newpage

\appendix
\section{Appendix}
\subsection{Proof of Theorem~\ref{P1}}
\label{App1}

Before detailing the proof, we first introduce two necessary lemmas.
\begin{lemm}
\label{L1}
$\forall$ K dimension continuous variable $X=(X_1,...,X_K)$, if the support set of its joint probability density is a convex set in $R^K$, then $\exists$ K dimensional independent uniform variables $U=(U_1,...,U_K)$ and a set of function $F=\{f_1,f_2,...,f_K\}$ result in $X=F(U)$ that $X_k=f_k(U_1,...,U_k)$.
\end{lemm}

The proof of Lemma~\ref{L1} is provided by~\citet{HePing}. 
The above lemma indicates that it is possible to represent continuous observations with convex joint density distributions by projecting a list of independent variables onto some functions.
{We now present the following lemma showing that any given continuous variable can be associated with a uniform distribution.}

\begin{lemm}
\label{L2}
$\forall$ continuous variable $X$, set its distribution function as $f(X)=P(X \leq x)$, then $f(X) \sim U(0,1)$. 
\end{lemm}

\begin{proof}
$P(f(X) \leq a)=P(X \leq f^{-1}(a))=f(f^{-1}(u))=a$, where $a$ is a constant.
\end{proof}

Then we prove Theorem~\ref{P1}:
\begin{proof}
Generally, proving Theorem~\ref{P1} is equal to finding functions $q_i$, $i=1,2,...,K$ that make $G_i=q_i(N^S_i)=Z^*_i, N^S \subseteq N$ true.
We now present the proof in four steps as follows.

Step 1: Because of Lemma\ref{L1}, we get that there exists independent uniform variables $(U_1,...,U_K)$ and function $F$ make $Z^*=F(U)$, in which:
\begin{equation}
    U=(U_1,...,U_K)
\end{equation}
\begin{equation}
\label{fi}
    Z^*_i=f_i(U_1,...,U_i)
\end{equation}

Set $U_i \sim U(a_i,b_i)$, $i=1,2,...,K$.

Step 2: Set there are arbitrary K independent normal variables $N_1,...,N_K$, $N_i \sim N(\mu_i,\sigma^2_i$). Set distribution function of $N_i$ is $g_i$, which means $g_i(x)=P(N_i \leq x)$. Denote:
\begin{equation}
\label{U'i}
    g_i(N_i)=U'_i, i=1,2,...,K
\end{equation}

Step 3: Because of Lemma\ref{L2}, $U'_1,...,U'_K$ are independent and they all are variables of uniform distribution $U(0,1)$.

Then $U_i$ in Step 1 can be represents by $U'_i$ as:
\begin{equation}
    U_i=(b_i-a_i)U'_i+a_i, i=1,2,...,K
\end{equation}
Because $g_i(N_i)=U'_i$ in (\ref{U'i}), which means $U'_i$ is function of $N_i$, so we can denote:
\begin{equation}
\label{hi}
    (b_i-a_i)U'_i+a_i=(b_i-a_i)g_i(N_i)+a_i=h_i(N_i)
\end{equation}

Step 4: Then we can find $Z^*_i=q_i(N_1,...,N_i)$ as:
\begin{equation}
\label{qi}
    Z^*_i=f_i(h_1(N_1),h_2(N_2),...,h_i(N_i))\\
    =q_i(N_1,...,N_i)\\
\end{equation}
In which set of $f_i$ is from equation \ref{fi} and $h_i$ is from equation \ref{hi}. Therefore, we can induce Theorem~\ref{P1} is right because $\forall$ $Z^*$ (in equation $ p^{*}(X)=\int_{Z^*}p^{*}(X|Z^*)p^{*}(Z^*)dZ^*$), $\exists$ a function $Q$ makes $G=Q(N^S)=\{q_i(N_1,...,N_i)\}_{i=1,2,...,K}$ (from equation (\ref{qi})) is equal to $Z^*$ .

\end{proof}

\subsection{Proof of Proposition~\ref{prop:tri_A}}
\label{App2}
\begin{proof}
	${}$ \\
	Let $A$ as a lower triangular matrix, $\hat{A}$ is permuted from $A$. We can induce that $\hat{A}*{diag(N)}$ can be acquired by such matrix with finite row exchange :
	\[
	\begin{pmatrix}  
	A_{11}N_{1}&{0}&{0}&...&{0}\\
	A_{21}N_1&A_{22}N_2&0&...&0\\
	A_{31}N_1&A_{32}N_2&A_{33}N_3&...&0\\
	...\\
	A_{K1}N_1&A_{K2}N_2&A_{K3}N_3&...&A_{KK}N_K
	\end{pmatrix}
	\]
	Where $\hat{A}=\{A_{ij}\}$. So we can induce that for each $q_i$ in Equation~\ref{qi}, there must exist one raw of $\hat{A}*{diag(N)}$ equal to $(a_{i1}N_1,a_{i2}N_2,...,a_{ii}N_i)$, which is exactly $q_i^{-1}(Z^*_i)$. So optimal generative factors $G=Z*$ can be expressed as $\hat{A}*{diag(N)}$. %

\end{proof}

Now we provide a remark that implies any functions $Q'$ with the mentioned two conditions is guaranteed to meet the same mapping relationships with the true function $Q$, which makes $G'$ acquired by such $Q$ has the same statistical properties with the true $G$. 
\begin{remark}
	\label{conditions}
	Consider a function set
	$Q'=\{Q'_1,...Q'_K\}$ that computes  causal generative factors with a row-wise formulation as $G_i=Q'_i(B_i)$.
	Assuming $B  = A * {diag(N')}\in \mathbb{R}^{K \times K}$, along with the following conditions: \\
	Condition 1:  $A \in \mathbb{R}^{K \times K}$  can be obtained by finite row exchanges of a lower triangular matrix;\\
	Condition 2: $ N'$ consists of $K$ independent normal variables as $ N'=(N'_1,N'_2,...,N'_K)^T$;\\
	then \\
	(1) There must exist two real numbers denoted as $a_i$ and $b_i$, and exist $N_k \in N$ as in Theorem \ref{P1}, such that the following equation holds: $N'_i=a_i\times N_k+b_i, i=1,2,...,K$.\\
	(2) Denote $Q=\{Q_1,Q_2,...,Q_K\}$(in Theorem \ref{P1}) and $Q_i^{-1}(G_i)=N^{S}_{i} \subseteq  \{N_1,N_2,...,N_K\}$, then for each $Q_i$ there must exist $Q'_m \in Q'$  and a constant diagonal matrix $V_i$ and constant vector $M_i$ such that: 
	$Q_i^{-1}(G_i)=V_i{Q'_m}^{-1}(G_i)+M_i=N^{S}_{i}$ if $Q_i$  and $Q'_m$ are reversible. 
\end{remark}

\begin{proof}
	Proof of Remark~\ref{conditions} is equal to proving such 2 statements:
	
	(1) $\forall i$ in 1,2,...,$K$, there exist a constant diagonal matrix $V_i$ and constant vector $M_i$ that $N'^S_i=(N'_1,N'_2,...,N'_i)$ can be acquired by $N^S_i=(N_1,N_2,...N_i)$: $N'^S_i=V_iN^S_i+M_i$.
	
	(2) $Q=\{q_1,q_2,...,q_K\}$ in Equation~\ref{qi} is unique.

	Proof of statement(1):\\
	Because $\forall$ two normal variables with the same dimensions, denoting $N_i$ and $N'_i$, there exist two constant $v_i,m_i$ that $N'_i=v_iN_i+m_i$. Statement(2) just represents $N'_i=v_iN_i+m_i$ as random vector.
	
	Proof of statement(2):\\
	Because $f_i$ and $h_i$ is unique. So $q_i$ is unique and $Q$ is unique.
\end{proof}
The (1) in Theorem~\ref{conditions} can be acquired by statement(1). Because statement(2) tells us $Q$ is unique, so we can induce by Proposition~\ref{prop:tri_A} that Equation~\ref{qi} can also be acquired by $G_i=Q_i(B_i)=Q'_i(B_i)$ if $N=N'$. However, we can not ensure $N=N'$, so with statement(2), we can induce (2) in Theorem~\ref{conditions}.

\subsection{Proof of Theorem~\ref{theo_linear}}
\label{App3}
Before proof, we need a lemma as following:
\begin{lemm}
	\label{L4}
	Denote the distribution function of normal variable $N \sim (\mu,\sigma^2)$ as $F_N(x)$. There exist a linear function $f_L(x)$: (1) Make the absolute error $|F_N(x)-f_L(x)| \le \frac{1}{\sqrt{2\pi } } x(1-e^{-\frac{x^2}{2\sigma^2} })$
	(2) If $x \in \left [ \mu-\varepsilon, \mu+\varepsilon\right ]$, then the absolute error $|F_N(x)-f_L(x)| \le \frac{1}{\sqrt{2\pi } } \varepsilon(1-e^{-\frac{\varepsilon^2}{2\sigma^2} })$
\end{lemm}
\begin{proof}
	The proof of (1):
	Without losing generalization, we consider $\mu=0$ and the linear function is $f_L(x)=\frac{1}{\sqrt 2\pi}x+\frac{1}{2}$. We only consider $x >0$, because both of $f_L(x)$ and distribution function are central symmetric about $(0,\frac{1}{2})$, so the absolute error is same when $x<0$.
	\begin{equation}
	\begin{split}
	& |F_N(t)-f_L(t)|\\
	& = (\frac{1}{\sqrt{2\pi } }t+\frac{1}{2 })-\int_{-\infty }^{t} \frac{1}{\sqrt{2\pi } } e^{-\frac{x^2}{2\sigma^2} }dx\\
	&=\frac{1}{\sqrt{2\pi } }t-\int_{0 }^{t} \frac{1}{\sqrt{2\pi } } e^{-\frac{x^2}{2\sigma^2} }dx\\
	& =\int_{0 }^{t}\frac{1}{\sqrt{2\pi } }dx-\int_{0 }^{t} \frac{1}{\sqrt{2\pi } } e^{-\frac{x^2}{2\sigma^2} }dx\\
	& =\int_{0 }^{t}\frac{1}{\sqrt{2\pi } }(1- e^{-\frac{x^2}{2\sigma^2} })dx\\
	& \le \int_{0 }^{t}\frac{1}{\sqrt{2\pi } }(1- e^{-\frac{t^2}{2\sigma^2} })dx\\
	& =\frac{1}{\sqrt{2\pi } }t(1- e^{-\frac{t^2}{2\sigma^2} })
	\end{split}
	\end{equation}

The proof of (2):
Without losing generalization, we consider $\mu=0$. Because $\frac{1}{\sqrt{2\pi } }t(1- e^{-\frac{t^2}{2\sigma^2} })$ is a increasing function. If $t \in \left [ 0, \varepsilon\right ]$, $\frac{1}{\sqrt{2\pi } }t(1- e^{-\frac{t^2}{2\sigma^2} })\le \frac{1}{\sqrt{2\pi } } \varepsilon(1-e^{-\frac{\varepsilon^2}{2\sigma^2} })$. The same as $t \in \left [ -\varepsilon, 0\right ]$.
\end{proof}

Now we prove Theorem~\ref{theo_linear}:
\begin{proof}
	Assume $Z^*_1,Z^*_2,...,Z^*_K$ is a linear uniform distribution means that:
	$P(Z^*_1)\sim U(\eta_1+r_{10},\eta_1+r_{11})$ and $P(Z^*_k|Z^*_1,...,Z^*_{k-1}) \sim U(\eta_k+r_{k0},\eta_k+r_{k1}),k=2,...,K$ with $\eta_k= \sum_{t=1}^{k-1} \omega _{kt}Z^*_t$, and $\omega _{k1},\omega _{k2},...,\omega _{k(k-1)},r_{k0},r_{k1}$ is real number. Then we can get that the conditional distribution when $Z^*_k$ is in its non-zero interval:
	\begin{equation}
	\label{cond_z}
	\begin{split}
	& P(Z^*_k\le z^*_k|Z^*_1=z^*_1,Z^*_2=z^*_2,...,Z^*_{k-1}=z^*_{k-1})\\
	& =\frac{1}{r_{k1}-r_{k0}}(z^*_k-r_{k0})+\omega _{k1}z^*_1+...+\omega _{k(k-1)}z^*_{k-1}
	\end{split}
	\end{equation}
	We can find that this conditional distribution is a linear function of $(z^*_1,...,z^*_k)$. Moreover, we can induce by Mathematical Induction that the joint distribution $P(Z^*_1,...,Z^*_k)=P(Z^*_k|Z^*_1,...,Z^*_{k-1})P(Z^*_1,...,Z^*_{k-1})$ is a constant when $(Z^*_1,...,Z^*_k)$ is in their non-zero interval.
	
	We can find the formulation of $f_i$ in Lemma~\ref{L1} in ~\citet{HePing} as:\\
	$f_i(x_1,x_2,...,x_k)=\frac{p_X(X_i\le x_i|X_1=x_1,X_2=x_2,...,X_{i-1}=x_{i-1})}{p_X(x_1,...,x_{k-1}
		)}$
	
	$p_X$ means the probability dense function of $X$ in Lemma~\ref{L1}, but in equation~\ref{fi}, $p_X$ means the dense function of $Z^*=\{Z^*_1,Z^*_2,...,Z^*_K\}$ in true data set. So we can induce from equation~\ref{cond_z} that, the numerator of $f_i(U_1,...,U_i)$ in equation~\ref{fi} is linear function of $U_1,...,U_i$ and denominator is a constant. Therefore, equation~\ref{fi} is a linear function of $U_1,...,U_i$, denoting as:
	\begin{equation}
	\label{f2a}
	\begin{split}
	& f_i(U_1,...,U_i)\\
	& =\rho_{i0}+\rho_{i1}U_1+...+\rho_{i(i-1)}U_{i-1}+\rho_{ii}U_{i} 
	\end{split}
	\end{equation}
	
	In equation~\ref{qi}, $h_i$ is $(b_i-a_i)g_i(N_i)+a_i$, where $g_i$ is the distribution of normal variable. Without losing generality, we set $b_i=1,a_i=0$, then we have $h_i=g_i$, which has the absolute error bound with a linear function as Lemma~\ref{L4}.
	Based on equation~\ref{f2a} and Lemma~\ref{L4}, we can get there exist a linear function $f_L(N_1,N_2,...,N_i)=\sum_{k=1}^{i}f_{Lk}(N_k)$ has the bound with and the optimal solution $G=Z^*_i$ when $N_i \in \left[\mu_i-\delta,\mu_i+\delta\right]$ because:
	\begin{equation}
	\begin{split}
	& |f_i(h_1(N_1),h_2(N_2),...,h_i(N_i))-f_L(N_1,N_2,...,N_i)|\\
	& =|\rho_{i0}+\rho_{i1}h_1(N_1)+...+\rho_{i(i-1)}h_{i-1}(N_{i-1})+\rho_{ii}h_i(N_{i})\\
	& -\sum_{k=1}^{i}f_{Lk}(N_k)|\\
	& \le |\sum_{k=1}^{i}(\rho_{ik}h_1(N_k)-f_{Lk}(N_k))|+|\rho_{i0}|\\
	& \le |\rho_{i0}|+\frac{\delta}{\sqrt{2\pi}} \sum_{k=1}^{i}  \rho_{ik}(1-e^{-\frac{\delta ^2}{2\sigma^2_k} })\\
	& =|\rho_{i0}|+\delta (\sum_{k=1}^{i} \frac{\rho_{ik}}{\sqrt{2\pi}}-\sum_{k=1}^{i}\frac{\rho_{ik}}{\sqrt{2\pi}}e^{-\frac{\delta ^2}{2\sigma^2_k} })\\
	& = |\rho_{i0}|+\delta (\sum_{k=1}^{i} \frac{\rho_{ik}}{\sqrt{2\pi}}-\sum_{k=1}^{i}\frac{\rho_{ik}}{\sqrt{2\pi}}(e^{-\frac{1}{2\sigma^2_k} })^{\delta ^2})\\
	& = a_{i}+\delta(b_{i}+\sum_{k=1}^{i}c_{ik}d_k^{\delta^2})
	\end{split}
	\end{equation}
	
\end{proof}

\subsection{Proof of Proposition~\ref{def:general_factor}}
\label{App4}
\begin{proof}

	Now we start from one of rows in $\hat{A} * {diag(N')}$ which is $(a_{i1}N_1,a_{i2}N_2,a_{i3}N_3,0,...,0)$(for simplicity, we denote this row is i-th row. Note that there must exist such a row because $\hat{A}$ is permuted from triangular matrix). In Theorem~\ref{theo_linear}, we proved that $Q'$ can be implemented as a linear function. So $G'_i=Q'_i(B_i)=Q'_i(\{\hat{A} * {diag(N')}\}_i)=Q'_i(a_{i1}N_1,a_{i2}N_2,a_{i3}N_3,0,...,0)$. So:\\
	\begin{equation}
	\begin{split}
	& Q_i(a_{i1}N_1,a_{i2}N_2,a_{i3}N_3,0,...,0) \\
	& =q_{i1}a_{i1}*N_1+q_{i2}a_{i1}*N_2+q_{i3}a_{i1}*N_3 \\
	& =(q_{i1}a_{i1},q_{i2}a_{i2},q_{i3}a_{i3},0...,0)*(N_1,N_2,N_3,...,N_K)^T \\
	& \longrightarrow(c_{i1},c_{i2},c_{i3},0...,0)*(N_1,N_2,N_3,...,N_K)^T\\
	& =\Tilde{A_i}*N'\\
	\end{split}
	\end{equation}
	
	Here $\Tilde{A_i}=(c_{i1},c_{i2},c_{i3},0...,0)$ and  $N'=(N_1,N_2,N_3,...,N_K)^T$. We can adopt this expression method to other rows. Finally, we can make such conclusion: $\Tilde{A}$ is a matrix with the same non-zero position as $\hat{A}$, means that $\Tilde{A}$ is also a matrix permuted from lower triangular matrix.
\end{proof}

\subsection{Formal version of Theorem~\ref{theo:truth} }
\label{App7}
Given $n$ observations $\{X^{(1)},X^{(2)},...,X^{(n)}\}$ sampled from the same distribution $p^*(X)$, along with their corresponding optimal generative factors  $\{Z^{(1)},Z^{(2)},...,Z^{(n)}\}$, 
a function $\Bar{Z^{(n)}}$ with these generative factors will converge to the same ground truth (GT) concept $\ddot{Z}$ as
$lim_{n\to \infty}\Bar{Z^{(n)}}=\ddot{Z}$, where $\Bar{Z^{(n)}}=(Z^{(1)}+Z^{(2)}+...+Z^{(n)})/n$.
\begin{proof}
	The conclusion can be easily deduced from the law of large numbers.
\end{proof}

\subsection{Proof of $(I-\Phi^T)^{-1}$}
\label{Proof3}
Before the proof, we need a lemma:
\begin{lemm}
	\label{L3}
	If $D$ is the adjacency matrix of DAG with nodes vector $Z$, which means that $Z=DZ$. Then there exist a lower triangular matrix $T$ and a vector $Z'$ acquired by finite row exchange of $Z$, making $Z'=TZ'$.
\end{lemm}
\begin{proof}
	In a DAG, there must exist at least one node with 0 in-degree. We can remove arbitrary one node with 0 in-degree, and make this node as the first node in $Z'$. 
	
	Because the graph without this node is also a DAG, so we can also find at least one node with 0 in-degree, and make the second node in $Z'$. Because the second node's in-degree is 0 in graph without the first node, so the first line of $T$ has at most 1 non-zero element. 
	
	Repeat this process and we can find such $T$ and $Z'$.   
\end{proof}

To prove the $(I-\Phi^T)^{-1}$ can be acquired by a lower triangular matrix with finite row exchange, we need to prove: \\
There exist a lower triangular matrix $L$ and a elementary matrix $P_{K \times K}$ acquired by unit matrix with finite row exchange, making that $G^T=(I-\Phi^T)^{-1}\epsilon^T=PL\epsilon^T$, $\Phi$ is DAG adjacency matrix with nodes vector $G^T$.\\
Now we have such proof:
\begin{proof}
	With Lemma~\ref{L3}, we can induce that there exist a lower triangular matrix $T$ and a vector $G_R^T$ acquired by finite row exchange of $G^T$, making $G_R^T=TG_R^T+\epsilon \longrightarrow G_R^T=(I-T)^{-1}\epsilon^T$.\\
	Because $G_R^T$ is acquired by finite row exchange of $G^T$, we can denote $G_R^T=P'G^T$, $P'$ is acquired by unit matrix with finite row exchange. So we have $P'G^T=(I-T)^{-1}\epsilon^T \longrightarrow G^T=P'^{-1}(I-T)^{-1}\epsilon^T$. Let $P=P'^{-1}$, $L=(I-T)^{-1}$, then we have $G^T=PL\epsilon^T$.
\end{proof}

\subsection{Algorithm of \ourmeth and Causal-Meta-Graph for Few Shot Link Prediction}
\label{App6}
\begin{algorithm}[!h] 
	\caption{Concept-free Causal-Meta-Graph for Few Shot Link Prediction}
	\KwResult{GNN global parameters $\phi$, Graph signature function $\psi$, Global causal layer parameters $C$}
	Initialize learning rates: $\alpha$, $\beta$, $\gamma$; 
	
	Sample a mini-batch of graphs, $G_{batch}$ from $p(G)$\\
	\For{each $G \in G_{batch}$ }
	{
		$\varepsilon = \varepsilon^{train} \cup \varepsilon^{val} \cup \varepsilon^{test}$// Split edges into train, val, and test;\;
		
		$s_G = \psi (G, \varepsilon^{train})$// Compute graph signature\;
		
		Initialize: $\phi (0)$ ← $\phi$// Initialize local parameters via global parameters\\
		
		\For{k in $[1:K]$}{
			$s_G = stopgrad(s_G)$ // Stop Gradients to Graph Signature\;
			
			$Z = (I-C^T)^{-1}s_G$// Compote hidden representation\;
			
			$L_{train}=-E L B O_{train}+\alpha H(C)$\;
			
			Update $\phi (k)$ ←$ \phi(k-1) - \beta \triangledown \phi L_{train}$
		}
		
		Initialize: $\phi$ ← $\phi_{K}$\;
		
		$s_G = \psi (G, \varepsilon_{val} \cup \varepsilon_{train}) $// Compute graph signature with validation edges\;
		
		$L_{val}=-E L B O_{val}+\alpha H(C)$\;
		
		Update $\phi$ ← $\phi - \gamma \triangledown \phi L_{val}$\;
		
		Update $\psi$ ← $\psi - \gamma \triangledown \psi L_{val}$\;
		
		Update $C$ ← $C - \gamma \triangledown C L_{val}$
	}
\end{algorithm}

\begin{algorithm}[!h] 
	\caption{\ourmeth}
	\KwIn{Graph edges $\mathcal{E}$, node features $X$, $\alpha$, $\beta$}
	Initialize GCN parameters $GCN_\mu$, $GCN_\sigma$, causal matrix $\Phi$;
	
	$\mathcal{E} = \mathcal{E}^{train} \cup \mathcal{E}^{val} \cup \mathcal{E}^{test}$// Split edges into train, val, and test;\\
	$A = A_{train}+A_{val}+A_{test}$//Generate train, valid, test adjacency\\
	
	\For{epoch in $[1: $ number of epoch$]$}{
		$\mu = GCN_\mu(A_{train},X)$// Compute mean of $\varepsilon$\;
		
		$\sigma = GCN_{\sigma}(A_{train},X)$// Compute variance of $\varepsilon$\;
		
		$\varepsilon = N(\mu,\sigma)$//Generate $\varepsilon$ as independent normal distribution\;
		
		$G = (I-C^T)^{-1}\varepsilon$// Compute generate factors\;
		
		${\hat{A}_{train}}=\sigma_1\left({G}_{i}^{T} \hat{G}_{j}\right) $//Reconstruct adjacency matrix\;
		
		$\hat{X}= \sigma_2{(G})$//Reconstruct node features\;
		
		$\mathcal{L}_{train}=-\mathcal{L}_G+\alpha \mathcal{L}_{\Phi}+\beta \mathcal{L}_{MSE}$\;
		
		Update $GCN_\mu$, $GCN_\sigma$, $\Phi$
	}
	Compute $ROC(A_{test}, \hat{A}_{test})$ and $AP(A_{test}, \hat{A}_{test})$
	
\end{algorithm}

\subsection{Ablation study result}
\begin{table}[!th]
	\label{ablation_MSE}
	\caption{
		AUC ($\%$) and AP ($\%$)  scores for \ourmeth w/o MSE and \ourmeth on real-world datasets. }
	\scalebox{0.9}{
		\begin{tabular}{l|cccc}
			\toprule
			
			\multirow{2}{*}{} & \multicolumn{2}{c}{\ourmeth w/o MSE}               & \multicolumn{2}{c}{CCVGAE}                \\
			& AUC                 & AP                  & AUC                 & AP                  \\ \midrule
			Cora              & 0.80${_{\pm 0.02}}$  & 0.82${_{\pm 0.03}}$ & {\bf0.85}${_{\pm 0.03}}$ & {\bf0.85}${_{\pm 0.05}}$ \\
			Corn              & 0.73${_{\pm 0.03}}$ & {\bf0.79}${_{\pm 0.06}}$ & {\bf0.74}${_{\pm 0.06}}$ & 0.78${_{\pm 0.04}}$ \\
			Texas             & {\bf0.76}${_{\pm 0.04}}$ & 0.78${_{\pm 0.03}}$ & 0.75${_{\pm 0.07}}$ & {\bf0.80}${_{\pm 0.07}}$ \\
			Wisconsin         & 0.72${_{\pm 0.07}}$ & 0.77${_{\pm 0.06}}$ & {\bf0.75}${_{\pm 0.04}}$ & {\bf0.79}${_{\pm 0.05}}$ \\
			dRisk             & 0.71${_{\pm 0.05}}$ & 0.65${_{\pm 0.07}}$ & {\bf0.75}${_{\pm 0.06}}$ & {\bf0.72}${_{\pm 0.05}}$ \\
			Actor             & {\bf0.79}${_{\pm 0.04}}$ & {\bf0.81}${_{\pm 0.03}}$ & 0.78${_{\pm 0.07}}$ & {\bf0.81}${_{\pm 0.06}}$ \\ 
			\bottomrule
		\end{tabular}
	}
\end{table}

	
	
	

\end{document}